\def\year{2022}\relax
\documentclass[letterpaper]{article} 
\usepackage{aaai22}  
\usepackage{times}  
\usepackage{helvet}  
\usepackage{courier}  
\usepackage[hyphens]{url}  
\usepackage{graphicx} 
\usepackage{natbib}  
\usepackage{caption} 
\DeclareCaptionStyle{ruled}{labelfont=normalfont,labelsep=colon,strut=off} 
\usepackage{algorithm}
\usepackage{algorithmic}

\usepackage{newfloat}
\usepackage{listings}
\floatstyle{ruled}
\newfloat{listing}{tb}{lst}{}
\floatname{listing}{Listing}
\title{Reward-Weighted Regression Converges to a Global Optimum}
\author{
   Miroslav {\v{S}}trupl,\textsuperscript{\rm 1}\thanks{Equal contribution. Correspondence to struplm@idsia.ch}
   Francesco Faccio,\textsuperscript{\rm 1}\textsuperscript{$*$}
   Dylan R.~Ashley,\textsuperscript{\rm 1} \\
   Rupesh Kumar Srivastava,\textsuperscript{\rm 2}
   J{\"{u}}rgen Schmidhuber\textsuperscript{\rm 1,2,3}
}
\affiliations{
   \textsuperscript{\rm 1} The Swiss AI Lab IDSIA, Universit{\`{a}} della Svizzera italiana (USI) \& SUPSI, Lugano, Switzerland \\
   \textsuperscript{\rm 2} NNAISENSE, Lugano, Switzerland \\
   \textsuperscript{\rm 3} King Abdullah University of Science and Technology (KAUST), Thuwal, Saudi Arabia \\
   \{struplm, francesco, dylan.ashley\}@idsia.ch, rupesh@nnaisense.com, juergen@idsia.ch
}

\usepackage{amsfonts}
\usepackage{amsmath}
\usepackage{amsthm}
\usepackage{bm}
\usepackage{thm-restate}
\usepackage{thmtools}
\usepackage{xfrac}
\DeclareMathOperator*{\argmax}{arg\,max}
\DeclareMathOperator*{\ev}{\mathbb{E}}

\DeclareMathOperator*{\Indicator}{\mathbf{1}}

\newcommand{\Var}{\mathrm{Var}}

\newcommand{\de}{\,\mathrm{d}}
\newcommand{\vrho}{\mathbr{\rho}}

\newcommand{\hyscoreprime}[1][\vtheta]{\nabla_{\vrho'}\log\nu_{\vrho'}}

\newcommand{\dif}{\mathrm{d}}

\declaretheorem[name=Lemma,numberwithin=section]{lemma}
\newtheorem{definition}{Definition}

\begin{document}

\maketitle

\begin{abstract}
Reward-Weighted Regression (RWR) belongs to a family of widely known iterative Reinforcement Learning algorithms based on the Expectation-Maximization framework. In this family, learning at each iteration consists of sampling a batch of trajectories using the current policy and fitting a new policy to maximize a return-weighted log-likelihood of actions. Although RWR is known to yield monotonic improvement of the policy under certain circumstances, whether and under which conditions RWR converges to the optimal policy have remained open questions. In this paper, we provide for the first time a proof that RWR converges to a global optimum when no function approximation is used, in a general compact setting. Furthermore, for the simpler case with finite state and action spaces we prove R-linear convergence of the state-value function to the optimum.
\end{abstract}

\section{Introduction}
\label{sec:introduction}

Reinforcement learning (RL) is a branch of artificial intelligence that considers learning agents interacting with an environment~\citep{Sutton:2018:RLI:3312046}. RL has enjoyed several notable successes in recent years. These include both successes of special prominence within the artificial intelligence community---such as achieving the first superhuman performance in the ancient game of Go~\citep{silver2016mastering}---and successes of immediate real-world value---such as providing autonomous navigation of stratospheric balloons to provide internet access to remote locations~\citep{bellemare2020autonomous}.

One prominent family of algorithms that tackle the RL problem is the Reward-Weighted Regression (RWR) family~\citep{peters2007reinforcement}. The RWR family is notable in that it naturally extends to continuous state and action spaces. The lack of this functionality in many methods serves as a strong limitation. This prevents them from tackling some of the more practically relevant RL problems---such as many robotics tasks \cite{plappert2018multi}. Recently, RWR variants were able to learn high-dimensional continuous control tasks~\citep{peng2019advantage}. RWR works by transforming the RL problem into a form solvable by well-studied expectation-maximization (EM) methods~\citep{dempster1977maximum}. EM methods are, in general, guaranteed to converge to a point whose gradient is zero with respect to the parameters. However, these points could be both local minima or saddle points~\citep{wu1983convergence}. These benefits and limitations transfer to the RL setting, where it has been shown that an EM-based return maximizer is guaranteed to yield monotonic improvements in the expected return~\citep{dayan1997using}. However, it has been challenging to assess under which conditions---if any---RWR is guaranteed to converge to the optimal policy. This paper presents a breakthrough in this challenge.

The EM probabilistic framework requires that the reward obtained by the RL agent is strictly positive, such that it can be considered as an improper probability distribution. Several reward transformations have been proposed, e.g., \citet{peters2007reinforcement,peters2008learning, peng2019advantage,abdolmaleki2018maximum}, frequently involving exponential transformations. In the past, it has been claimed that a positive, strictly increasing transformation $u_{\tau}(s)$ with $\int_0^{\infty}u_{\tau}(r) \de r = \textit{const}$ would not alter the optimal solution for the MDP~\citep{peters2007reinforcement}. Unfortunately, as we demonstrate in Appendix~\ref{ap:CEsection}, this is not the case. The consequence of this is that we cannot rely on those transformations if we want prove convergence. Therefore, we consider only linear transformation of the reward. A possible disadvantage of relying on linear transformations is that it is necessary to know a lower bound on the reward to construct such a transformation.

In this work, we provide the first proof of RWR's global convergence in a setting without function approximation or reward transformations\footnote{Note that---without loss of generality---we do assume here that a linear reward transformation is already provided, such that the reward is positive}. The paper is structured as follows: Section~\ref{sec:background} introduces the MDP setting and other preliminary material; Section~\ref{sec:rwr} presents a closed-form update for RWR based on the state and action-value functions and Section~\ref{sec:mit} shows that the update induces monotonic improvement related to the variance of the action-value function with respect to the action sampled by the policy; \textbf{Section~\ref{sec:conv} proves global convergence of the algorithm in the general compact setting and convergence rates in the finite setting}; Section~\ref{sec:exp} illustrates experimentally that---for a simple MDP---the presented update scheme converges to the optimal policy; Section~\ref{sec:related_work} discusses related work; and Section~\ref{sec:conclusion} concludes.

\section{Background}
\label{sec:background}

Here we consider a Markov Decision Process (MDP)~\citep{stratonovich1960,puterman2014markov}
$\mathcal{M}=(\mathcal{S},\mathcal{A},p_T,R,\gamma,\mu_0)$. We assume that the state and action spaces $\mathcal{S}\subset\mathbb{R}^{n_S}$,
$\mathcal{A}\subset\mathbb{R}^{n_A}$ are compact sub-spaces \footnote{This allows for state and action vectors that have discrete, continuous, or mixed components.}
(equipped with subspace topology), with measurable structure
given by measure spaces
$(\mathcal{S},\mathcal{B}(\mathcal{S}),\mu_S)$,
$(\mathcal{A},\mathcal{B}(\mathcal{A}),\mu_A)$
where $\mathcal{B}(\cdot)$ denotes the
Borel $\sigma$-algebra after completion, and reference measures $\mu_S$, $\mu_A$ are assumed to be finite and strictly positive on $\mathcal{S},\mathcal{A}$ respectively.
The distributions of state (action) random variables (except in Section~\ref{sec:conv} where greedy policies are used) are assumed to be dominated by $\mu_S$ ($\mu_A$), thus
having a density with respect to $\mu_S$ ($\mu_A$).
Therefore, we reserve symbols $\dif s, \dif a$ in integral expression not to integration with respect to Lebesgue measure, as usual, but to integration with respect to $\mu_S$ and $\mu_A$ respectively, e.g. $\int_{\mathcal{S}} (\cdot) \dif s := \int_{\mathcal{S}} (\cdot) \dif \mu_S(s) $.
Let $(\Omega,\mathcal{F},\mu)$ be a measure space
and $f:\Omega \rightarrow \mathbb{R}^+$ a $\mathcal{F}$ measurable function
(density). We denote by $f\cdot\mu$ the measure which assigns to every set
$B \in \mathcal{F}$ a measure $f\cdot\mu(B) := \int_B f \dif \mu$.

In the MDP framework, at each step, an agent observes a state $s \in \mathcal{S}$, chooses an action $a \in \mathcal{A}$, and subsequently transitions into state $s'$ with probability density $p_T(s' | s, a)$
to receive a deterministic reward $R(s,a)$.
The transition probability kernel is assumed to be continuous in total variation in $(s,a) \in \mathcal{S}\times \mathcal{A}$ (the product topology is assumed on $\mathcal{S}\times \mathcal{A}$), and thus the density $p_T(s' | s, a)$ is continuous (in $\|\cdot\|_1$ norm). $R(s,a)$ is assumed to be a continuous function on $\mathcal{S}\times \mathcal{A}$.

The agent starts from an initial state (chosen under a probability density $\mu_0(s)$) and is represented by a stochastic policy $\pi$: a probability kernel which provides the conditional probability distribution of performing action $a$ in state $s$.\footnote{In Sections~\ref{sec:rwr} and~\ref{sec:mit}, a policy is given through its conditional density with respect to $\mu_A$. We also refer to this density as a policy.} The policy is deterministic if, for each state $s$, there exists an action $a$ such that $\pi(\{a\}|s) = 1$. The return $R_t$ is defined as the cumulative discounted reward from time step t:
$R_t =  \sum_{k=0}^{\infty}\gamma^k R(s_{t+k+1}, a_{t+k+1})$
where $\gamma \in (0,1)$ is a discount factor. We discuss the undiscounted case ($\gamma=1$) in Appendix \ref{ap:MITsection}, which covers the scenario with absorbing states.

The agent's performance is measured by the cumulative discounted expected reward (i.e., the expected return), defined as  $J(\pi)={\ev}_{\pi}[R_0].$ The state-value function $V^{\pi}(s) = {\ev}_{\pi}[R_t|s_t=s]$ of a policy $\pi$ is defined as the expected return for being in a state $s$ while following $\pi$. The maximization of the expected cumulative reward can be expressed in terms of the state-value function by integrating it over the state space $\mathcal{S}$: $J(\pi)=\int_{\mathcal{S}} \mu_0(s) V^{\pi}(s) \de s$.
The action-value function $Q^{\pi}(s,a)$---defined as the expected return for performing action $a$ in state $s$ and following a policy $\pi$---is $Q^{\pi}(s,a) = {\ev}_{\pi}[R_t|s_t=s,a_t=a]$. State and action value functions are related by $V^{\pi}(s) = \int_{\mathcal{A}} \pi(a|s) Q^{\pi}(s,a)  \de a$. We define as $d^{\pi}(s')$ the discounted weighting of states encountered starting at $s_0 \sim \mu_0(s)$ and following the policy $\pi$: $d^{\pi}(s') = \int_{\mathcal{S}}\sum_{t=1}^{\infty} \gamma^{t-1} \mu_0(s) p_{s_t|s_0,\pi}(s'|s) \de s $, where $p_{s_t|s_0,\pi}(s'|s)$ is the probability density of transitioning to $s'$ after t time steps, starting from $s$ and following policy $\pi$. We assume that the reward function $R(s,a)$ is strictly positive\footnote{It is enough to assume that the reward is bounded, so it can be linearly mapped to a positive value.}, so that state and action value functions are also bounded $V^{\pi}(s)\leq \frac{1}{1-\gamma}||R||_{\infty} = B_V < +\infty$. We define the operator\footnote{W maps to continuous functions since $R(s,a)$ is continuous and continuity of the integral follows from continuity of $p_T$ in $\|\cdot\|_1$ norm and boundedness of $V$.}
$W :L_{\infty}(\mathcal{S}) \rightarrow C(\mathcal{S}\times \mathcal{A})$
as $[W(V)](s,a) :=  R(s,a)+ \gamma \int_{\mathcal{S}} V(s') p_T(s'|s,a) \dif s'$ and the Bellman's optimality operator $T :L_{\infty}(\mathcal{S}\times \mathcal{A}) \rightarrow C(\mathcal{S}\times \mathcal{A})$ as $[T(Q)](s,a) :=  R(s,a) + \gamma \int_{\mathcal{S}}\max_{a'} Q(s',a') p_T(s'|s,a) \dif s'$. An action-value function $Q^{\pi}$ is optimal if it is the unique fixed point for $T$. If $Q^{\pi}$ is optimal, then $\pi$ is an optimal policy.

\section{Reward-Weighted Regression}
\label{sec:rwr}

Reward-Weighted Regression (RWR, see \citep{dayan1997using},\citep{peters2007reinforcement},\citep{peng2019advantage}) is an iterative algorithm which consists of two main steps. First, a batch of episodes is generated using the current policy $\pi_n$  (all policies in this section are given as conditional densities with respect to $\mu_A$). Then, a new policy is fitted to (using supervised learning under maximum likelihood criterion) a sample representation of the conditional distribution of an action given a state, weighted by the return. The RWR optimization problem is:
\begin{multline}
    \pi_{n+1} = \argmax_{\pi \in \Pi} \ev_{s \sim d^{\pi_n}(\cdot), a \sim \pi_{n}(\cdot|s)} \big[ \\ \ev_{R_t \sim p(\cdot|s_t = s, a_t = a, \pi_n)} \left[ R_t \log \pi(a|s) \right] \big],
\end{multline}
where $\Pi$ is the set of all conditional probability densities (meant with respect to $\mu_A$)\footnote{We can restrict to talk about probability kernels dominated by $\mu_A$ instead of all probability kernels thanks to Lebesgue decomposition.}.
Notice that $\pi_{n+1}$ is defined correctly as its expression does not depend on $t$.
This is equivalent to the following:
\begin{equation}
    \pi_{n+1} = \argmax_{\pi \in \Pi} \ev_{s \sim d^{\pi_n}(\cdot), a \sim \pi_{n}(\cdot|s)} \left [Q^{\pi_n}(s,a) \log \pi(a|s) \right].
\end{equation}
We start by deriving a closed form solution to the optimization problem:
\begin{restatable}[]{thr}{update}
\label{th:multbyQ}
Let $\pi_0$ be an initial policy and let $\forall s \in \mathcal{S}, \forall a \in \mathcal{A}$ $R(s,a)>0$. At each iteration $n>0$, the solution of the RWR optimization problem is:
\begin{equation}
    \pi_{n+1}(a|s) = \frac{Q^{\pi_n}(s,a) \pi_{n}(a|s)}{V^{\pi_n}(s)}.
\end{equation}
\end{restatable}
\begin{proof}
\begin{multline*}
    \pi_{n+1} = \argmax_{\pi \in \Pi} \int_{\mathcal{S}} d^{\pi_n}(s)
    \\ \times
    \int_{\mathcal{A}} \pi_{n}(a|s) Q^{\pi_n}(s,a) \log \pi(a|s) \de a \de s.
\end{multline*}
Define $\hat{f}(s,a) :=d^{\pi_n}(s) \pi_{n}(a|s) Q^{\pi_n}(s,a) $. $\hat{f}(s,a)$ can be normalized such that it becomes a density that we fit by $\pi_{n+1}$:
\begin{align*}
    f(s,a) &= \frac{\hat{f}(s,a)}{ \int_{\mathcal{S}} \int_{\mathcal{A}} \hat{f}(s,a) \de a \de s} \\ &= \frac{d^{\pi_n}(s) \pi_{n}(a|s) Q^{\pi_n}(s,a)}{\int_{\mathcal{S}} \int_{\mathcal{A}} d^{\pi_n}(s) \pi_{n}(a|s) Q^{\pi_n}(s,a) \de a \de s}.
\end{align*}
For the function to be maximized we have:
\begin{multline*}
    \int_{\mathcal{S}} \int_{\mathcal{A}}f(s,a) \log \pi(a|s) \de a \de s =
    \\
    \begin{aligned}
    &= \int_{\mathcal{S}} f(s) \int_{\mathcal{A}}f(a|s) \log \pi(a|s) \de a \de s \\
    &\leq  \int_{\mathcal{S}} f(s) \int_{\mathcal{A}}f(a|s) \log f(a|s) \de a \de s,
    \end{aligned}
\end{multline*}
where the last inequality holds for any policy $\pi$, since $\forall s \in \mathcal{S}$ we have that $ \int_{\mathcal{A}}f(a|s) \log \pi(a|s) \de a \leq  \int_{\mathcal{A}}f(a|s) \log f(a|s) \de a$, as $f(a|s)$ is the maximum likelihood fit. Note that for all states $s \in \mathcal{S}$ such that $d^{\pi_n}(s)=0$, we have that $f(s,a)=0$. Therefore, for such states, the policy will not contribute to the objective and can be defined arbitrarily. Now, assume $d^{\pi_n}(s)>0$. The objective function achieves a maximum when the two distributions are equal:
\begin{multline*}
    \pi_{n+1}(a|s) = f(a|s) = \frac{f(s,a)}{f(s)} =  \frac{f(s,a)}{\int_{\mathcal{A}}f(s,a) \de a} = \\
    \begin{aligned}
    & = \frac{d^{\pi_n}(s) \pi_{n}(a|s) Q^{\pi_n}(s,a)}{\int_{\mathcal{S}} \int_{\mathcal{A}} d^{\pi_n}(s) \pi_{n}(a|s) Q^{\pi_n}(s,a) \de a \de s}
    \\ 
    & \quad\cdot \frac{\int_{\mathcal{S}} \int_{\mathcal{A}} d^{\pi_n}(s) \pi_n(a|s) Q^{\pi_n}(s,a) \de a \de s}{\int_{\mathcal{A}} d^{\pi_n}(s) \pi_{n}(a|s) Q^{\pi_n}(s,a) \de a}
    \\
    & = \frac{\pi_{n}(a|s) Q^{\pi_n}(s,a)}{\int_{\mathcal{A}} \pi_n(a|s) Q^{\pi_n}(s,a) \de a} = \frac{Q^{\pi_n}(s,a) \pi_n(a|s) }{V^{\pi_n}(s)}.
    \end{aligned}
\end{multline*}
We can now set $\pi_{n+1}(a|s) = \frac{Q^{\pi_n}(s,a) \pi_{n}(a|s)}{V^{\pi_n}(s)}$ also for all $s$ such that $d^{\pi_n}(s)=0$, which completes the proof. Note that $V^{\pi_n}(s)$ is positive thanks to the assumption of positive rewards. Similarly, the denominator $\int_{\mathcal{S}} \int_{\mathcal{A}} \hat{f}(s,a) \de a \de s = \int_{\mathcal{S}}  d^{\pi_n}(s) V^{\pi_n}(s) \de s > 0$ is positive.
\end{proof}
When function approximation is used for policy $\pi$, the term $f(s)$ weighs the mismatch between $\pi(a| s)$ and $f(a| s)$. Indeed, we have $f(s) \propto d^{\pi}(s)V^{\pi}(s) $, suggesting that the error occurring with function approximation would be weighted more for states visited often and with a bigger value. In our setting, however, the two terms are equal since no function approximation is used.\\
Theorem \ref{th:multbyQ} provides us with an interpretation on how the RWR update rule works: at each iteration, given a state $s$, the probability over an action $a$ produced by policy $\pi_{n}$ will be weighted by the expected return obtained from state $s$, choosing action $a$ and following $\pi_n$. This result will be then normalized by $V^{\pi_n}(s)$, providing a new policy $\pi_{n+1}$. Alternatively, we can interpret this new policy as the fraction of return obtained by policy $\pi_n$ from state $s$, after choosing action $a$ with probability $\pi_n(\cdot|s)$. Intuitively, assigning more weight to actions which lead to better return should improve the policy. We prove this in the next section.

\section{Monotonic Improvement Theorem}
\label{sec:mit}

Here we prove that the update defined in Theorem \ref{th:multbyQ} leads to monotonic improvement.\footnote{The case where the MDP has non-negative rewards
and the undiscounted case are more complex and
treated in Appendix~\ref{ap:MITsection}.}
\begin{restatable}[]{thr}{mit}
\label{th:mit}
Fix $n > 0$ and let $\pi_0 \in \Pi$ be a policy\footnote{Also in this section all policies are given as conditional densities with respect to $\mu_A$.}. Assume $\forall s \in \mathcal{S}, \forall a \in \mathcal{A}$, $R(s,a)>0$. Define the operator $B : \Pi \rightarrow \Pi$ such that
$B(\pi) := \frac{Q^{\pi}\pi}{V^{\pi}}$ for $\pi \in \Pi$. Thus
$\pi_{n+1} = B(\pi_n)$, i.e. $\forall s \in \mathcal{S},
\forall a \in \mathcal{A}:\:\pi_{n+1}(a|s) = (B\pi_n)(a|s) = \frac{Q^{\pi_n}(s,a)\pi_n(a|s)}{V^{\pi_n}(s)}$.
Then $\forall s \in \mathcal{S},
\forall a \in \mathcal{A}$ we have that $ V^{\pi_{n+1}}(s) \geq V^{\pi_n}(s)$ and $ Q^{\pi_{n+1}}(s,a) \geq Q^{\pi_n}(s,a)$. Moreover, if for some $s \in \mathcal{S}$ holds $\Var_{a \sim \pi_n(a|s)} [Q^{\pi_n}(s,a)] > 0$ then the first inequality above is strict, i.e.
$ V^{\pi_{n+1}}(s) > V^{\pi_n}(s)$.
\end{restatable}
\begin{proof}
We start by defining a function $V^{\pi_{n+1}, \pi_n}(s)$ as the expected return for using policy $\pi_{n+1}$ in state $s$ and then following policy $\pi_n$: $V^{\pi_{n+1}, \pi_n}(s):= \int_{\mathcal{A}} \pi_{n+1}(a|s) Q^{\pi_n}(s,a) \de a$. By showing that $\forall s \in \mathcal{S}$,  $V^{\pi_{n+1}, \pi_n}(s) \geq V^{\pi_n}(s)$, we get that $\forall s \in \mathcal{S}$, $V^{\pi_{n+1}}(s) \geq V^{\pi_n}(s)$.
\footnote{
The argument is the same as given in \citep{puterman2014markov}, see section on Monotonic Policy Improvement.
}
\\
Now, let $s$ be fixed:
\begin{align*}
    & V^{\pi_{n+1}, \pi_n}(s) \geq V^{\pi_n}(s) \\
    \Leftrightarrow & \int_{\mathcal{A}} \pi_{n+1}(a|s)Q^{\pi_n}(s,a) \de a \geq \int_{\mathcal{A}}\pi_n(a|s)Q^{\pi_n}(s,a) \de a\\
    \Leftrightarrow & \int_{\mathcal{A}} \frac{\pi_n(a|s)Q^{\pi_n}(s,a)^2}{V^{\pi_n}(s)} \de a \geq \int_{\mathcal{A}}\pi_n(a|s)Q^{\pi_n}(s,a) \de a\\
    \Leftrightarrow & \int_{\mathcal{A}} \pi(a|s)Q^{\pi_n}(s,a)^2 \de a \geq \Big (\int_{\mathcal{A}}\pi_n(a|s)Q^{\pi_n}(s,a) \de a\Big)^2\\
    \Leftrightarrow & \ev_{a \sim \pi_n(a|s)} [Q^{\pi_n}(s,a)^2] \geq \ev_{a \sim \pi_n(a|s)} [Q^{\pi_n}(s,a)]^2\\
    \Leftrightarrow & \Var_{a \sim \pi_n(a|s)} [Q^{\pi_n}(s,a)] \geq 0,
\end{align*}
which always holds. Finally, $ \forall s \in \mathcal{S}$, $\forall a \in \mathcal{A}$:
\begin{multline*}
    Q^{\pi_{n+1}}(s,a) = \\
    \begin{aligned}
    &= R(s,a) + \gamma \int_{\mathcal{S}}p_T(s'|s,a) V^{\pi_{n+1}}(s')\de s'
    \\
    &\geq R(s,a) + \gamma \int_{\mathcal{S}}p_T(s'|s,a) V^{\pi_n}(s')\de s'
    \\
    &= Q^{\pi_n}(s,a).
    \end{aligned}
\end{multline*}
\end{proof}
Theorem~\ref{th:mit} provides a relationship between the improvement in the state-value function and the variance of the action-value function with respect to the actions sampled. Note that if at a certain point the policy becomes deterministic---or it becomes the greedy policy of its action-value function (i.e. the optimal policy)---, then the operator B will map the policy to itself and there will be no improvement.

\section{Convergence Results}
\label{sec:conv}

\subsection{Weak convergence in topological factor}
It is worth discussing what type of convergence we can achieve by iterating the $B$-operator $\pi_n := B(\pi_{n-1})$, where $\pi_n$ are probability densities with respect to a fixed reference measure $\mu_A$.

Consider first the classic "continuous" variable case, where $\mu_A$ is the Lebesgue measure and fix $s \in S$.
Optimal policies are known to be greedy on the optimal action-value function
$Q^*(s,a)$. That is,  they concentrate all mass on $\argmax_a Q^*(s,a)$. If $\argmax_a Q^*(s,a)$ consists of just a single point $\{a^*\}$, then the optimal policy (measure), $\pi^*(\cdot|s)$ for $s$, concentrates all its mass in $\{a^*\}$. This means that the optimal policy does not have a density with respect to the Lebesgue measure. Furthermore $(\pi_n(\cdot|s)\cdot\mu_A) (\{a^*\}) = \int_{\{a^*\}} \pi_n(a|s) \dif \mu_A(a) = 0$, while $\pi^*(\{a^*\}|s) = 1$. However, we still want to show that the measures $\pi_{n}(\cdot|s)\cdot \mu_A$ get concentrated in the neighbourhood of $a^*$ and that this neighbourhood gets tinier as $n$ increases. We will use the concept of weak convergence to prove this.

Another problem arises when considering the above: since $\argmax_a Q^*(s,a)$ can consist of multiple points, the set of optimal policies is $\mathcal{P}(\argmax_a Q^*(s,a))$, where $\mathcal{P}(F) := \{\mu :\mu\:\text{is a probability measure on}\:\mathcal{B}(\mathcal{A}), \mu(F) = 1\}$ for a $F \in \mathcal{B}(\mathcal{A})$. We want to prove convergence even when the sequence of policies $\pi_n$ oscillates near $\mathcal{P}(\argmax_a Q^*(s,a))$. A way of coping with this is to make
$\argmax_a Q^*(s,a)$ a single point through topological factorisation, to obtain the limit by working in a quotient space.
The notion of convergence we will be using is described in the following definition.

\begin{definition}
\label{de:Rweakconv}
(Weak convergence of measures in metric space relative to a compact set)
Let $(X,d)$ be a metric space, $F \subset X$ a compact subset, $\mathcal{B}(X)$
its Borel $\sigma$-algebra.
Denote $(\tilde{X},\tilde{d})$
a metric space resulting as a topological quotient with respect to $F$ and $\nu$ the quotient
map $\nu:X \rightarrow \tilde{X}$ (see Lemma \ref{le:quotient} for details).
A sequence of probability measures $P_n$ is said to converge weakly
relative to $F$ to a measure $P$ denoted
$$
P_n \rightarrow^{w(F)} P,
$$
if and only if the image measures of $P_n$ under $\nu$ converge weakly to the
image measure of $P$ under $\nu$:
$$
\nu P_n \rightarrow^{w} \nu P.
$$
\end{definition}
Note that the limit is meant to be unique just in quotient space, thus if $P$
is a  weak limit (relative to $F$) of a sequence $(P_n)$, then also
all measures $P'$ for which $\nu P' = \nu P$ are relatively weak limits, i.e. $P'|_{\mathcal{B}(X)\cap F^c} = P|_{\mathcal{B}(X)\cap F^c}$.
Thus, they can differ on $\mathcal{B}(X)\cap F$. While the total mass assigned to $F$
must be the same for $P$ and $P'$, the distribution of masses inside $F$ may differ.
\subsection{Main results}
Consider for all $n>0$ the sequence generated by $\pi_n := B(\pi_{n-1})$.
For convenience, for all $n\geq 0$, we define  $Q_n := Q_{\pi_n},\quad V_n := V_{\pi_n}$.
First we note that, since the reward is bounded, the monotonic sequences of value functions converge point-wise to a limit:
\begin{align*}
(\forall s \in \mathcal{S})&: V_n(s) \nearrow V_L(s) \leq B_V < +\infty\\
(\forall s \in \mathcal{S}, a \in \mathcal{A} )&:Q_n(s,a) \nearrow Q_L(s,a) \leq B_V < +\infty,
\end{align*}
where $B_V = \frac{1}{1-\gamma}||R||_{\infty}$. Further $\forall n$ $Q_n$ is continuous since $Q_n = W(V_n)$ and $W$ maps all bounded functions to continuous functions.

The convergence proof proceeds in four steps:
\begin{enumerate}
    \item First we show in Lemma~\ref{le:bellwl} that $Q_L$ can be expressed in terms of $V_L$ through $W$ operator. This helps when showing that $Q_n$ converges uniformly to $Q_L$.
    \item Then we demonstrate in Lemma~\ref{le:greedylim} that $\forall s\in \mathcal{S}$ the sequence of policy measures $\pi_{n}(\cdot|s)\cdot\mu_A$ converges weakly relative to the set $M(s) :=\argmax_a Q_L(s,a)$ to a measure that assigns all probability mass to greedy actions of $Q_L(\cdot,s)$, i.e. $\pi_n(\cdot|s)\cdot\mu_A \rightarrow^{w(M(s))} \pi_L(\cdot|s) \in \mathcal{P}(M(s))$.
    However we are interested just in those $\pi_L$ which are kernels, i.e.
    $\pi_L \in \Pi_L := \{ \pi_L' : \pi_L'\:\text{is a probability kernel from $(\mathcal{S},\mathcal{B}(\mathcal{S}))$ to $(\mathcal{A},\mathcal{B}(\mathcal{A}))$} ,$
    $\forall s\in \mathcal{S}, \: \pi_L'(.|s) \in  \mathcal{P}(M(s))\}$ --- the set of all greedy policies on $Q_L$.
    \item At this point we do not know yet if $Q_L$ and $V_L$ are the value functions of $\pi_L$. We prove this in Lemma~\ref{le:Bseqprop} (together with previous Lemmas) by showing that they are fixed points of the Bellman operator.
    \item Finally, we state the main results in Theorem~\ref{th:CT}. Since $V_L$ and $Q_L$ are value functions for $\pi_L$ and $\pi_L$ is greedy with respect to $Q_L$, then $Q_L$ is the unique fixed point of the Bellman's optimality operator:
    \begin{multline*}
    Q_L(s,a)= [T(Q)](s,a) = \\
    = R(s,a) + \gamma \int_{\mathcal{S}}\max_{a'} Q(s',a') p_T(s'|s,a) \dif s'.
    \end{multline*}
    Therefore $Q_L$ and $V_L$ are optimal value functions and $\pi_L$ is an optimal policy for the MDP.
\end{enumerate}

\begin{lemma}\label{le:bellwl}
The following holds:\\
\textbf{1.} $Q_L = W(V_L)$,\\
\textbf{2.} $Q_L$ is continuous,\\
\textbf{3.} $Q_n$ converges to $Q_L$ uniformly.
\end{lemma}
\begin{proof}
\textbf{1.} Fix $(s,a) \in \mathcal{S}\times \mathcal{A}$. We aim to show $Q_L(s,a) - [W(V_L)](s,a) = 0$.
Since $Q_n = W(V_n)$, we can write:
\begin{multline*}
Q_L(s,a) - [W(V_L)](s,a) = \\
\begin{aligned}
&=
Q_L(s,a) - Q_n(s,a)
\\&\quad\quad\quad\quad
- [W(V_L)](s,a) + [W(V_n)](s,a)
\\
&\leq
| Q_L(s,a) - Q_n(s,a) |
\\&\quad\quad\quad\quad
+ |[W(V_L)](s,a) - [W(V_n)](s,a)|.
\end{aligned}
\end{multline*}
The first part can be made arbitrarily small as $Q_n(s,a) \rightarrow Q_L(s,a)$. Consider the second part and fix $\epsilon > 0$. Since $V_n \rightarrow V_L$ point-wise, from Severini-Egorov's theorem~\citep{severini} there exists $S_{\epsilon} \subset S$ with $(p_T(\cdot|s,a) \cdot \mu_S)(S_{\epsilon}^c) < \epsilon$ such that
$\|V_n -V_L \|_{\infty} \rightarrow 0$ on $S_{\epsilon}$. Thus there exists $n_0$ such that
$\|V_n -V_L \|_{\infty} < \epsilon$ for all $n > n_0$. Now let us rewrite the second part for $n > n_0$:
\begin{multline*}
|[W(V_L)](s,a) - [W(V_n)](s,a)|
\\
\begin{aligned}
&\leq
\int_{\mathcal{S}} |V_L(s') - V_n(s')| p_T(s'|s,a) \dif \mu_S(s')
\\
&=
\int_{S_{\epsilon}} |V_L(s') - V_n(s')| p_T(s'|s,a) \dif \mu_S(s')\\
&\quad+
\int_{S_{\epsilon}^c} |V_L(s') - V_n(s')|p_T(s'|s,a) \dif \mu_S(s')
\\
&\leq
\|V_L - V_n\|_{\infty}  + B_V \int_{S_{\epsilon}^c}  p_T(s'|s,a)  \dif \mu_S(s')
\\
&\leq
\epsilon + B_V \epsilon,
\end{aligned}
\end{multline*}
which can be made arbitrarily small.\\
\textbf{2.} $Q_L$ is continuous because $W$ maps all bounded measurable functions to continuous functions.\\
\textbf{3.} Since $Q_n$ and $Q_L$ are continuous functions in a compact space and $Q_n$ is a monotonically increasing sequence that converges point-wise to $Q_L$, we can apply Dini's theorem (see Th. 7.13 on page 150 in~\cite{rudin1976principles}) which ensures uniform convergence of $Q_n$ to $Q_L$.
\end{proof}
\begin{lemma}
\label{le:greedylim}
Let $\pi_n$ be a sequence generated by $\pi_n := B(\pi_{n-1})$. Let $\pi_0$ be continuous in actions and $\forall s \in \mathcal{S}$, $\forall a \in \mathcal{A}$, $\pi_0(a|s)>0$. Define $M(s):= \argmax Q_L(\cdot|s)$. Then $\forall \pi_L \in \Pi_L \neq \emptyset$,  $\forall s \in \mathcal{S}$, we have $\pi_n(\cdot|s)\cdot\mu_A \rightarrow^{w(M(s))} \pi_L(\cdot|s) (\in \mathcal{P}(M(s)))$.
\end{lemma}
\begin{proof}
First notice that the set $\Pi_L$ is nonempty\footnote{The argument goes as follows: $H:=\cup_{s\in S}\{s\}\times M(s)$ is a closed set, then $f(s):=\sup M(s)$ is upper semi-continuous and therefore measurable. Then graph of $f$ is measurable so we can define a probability kernel $\pi_L(B|s):=\Indicator_{B}(f(s))$ for all $B$ measurable.}.
Fix $\pi_L \in \Pi_L$ and $s \in \mathcal{S}$. In order to prove that $\pi_n(\cdot|s)\cdot\mu_A \rightarrow^{w(M(s))} \pi_L(\cdot|s)$, we will use a characterization of relative weak convergence that follows from an adaptation of the Portmanteau Lemma~\citep{billingsley2013convergence} (see Appendix~\ref{le:RPort}). In particular, it is enough to show that for all open sets $U \subset \mathcal{A}$ such that $U \cap M(s) = \emptyset$ or such that $M(s) \subset U$, we have that $\liminf_n (\pi_n(\cdot|s)\cdot\mu_A) U \geq \pi_L(\cdot|s) U$.

The case $U \cap M(s) = \emptyset$ is trivial since $\pi_L(\cdot|s)(U) = 0$. For the remaining case $M(s) \subset U$ it holds
$\pi_L(\cdot|s)(U)=1$. Thus we have to prove $\liminf_n (\pi_n(\cdot|s)\cdot \mu_A) U = 1$. If we are able to construct an open set $D \subset U$ such that  $(\pi_n(\cdot|s)\cdot \mu_A)(D) \rightarrow 1$ for $n \rightarrow \infty$, then we will get that $\liminf_n (\pi_n(\cdot|s)\cdot\mu_A) U \geq 1$, satisfying the condition for relative weak convergence of $\pi_n(\cdot|s)\cdot\mu_A \rightarrow^{w(M(s))} \pi_L(\cdot|s)$.

The remainder of the proof will focus on constructing such a set. Fix $a^* \in M(s)$ and $0 < \epsilon < 1/3$.
Define a continuous map $\lambda: \mathcal{A} \rightarrow \mathbb{R}^+$ and closed sets $A_{\epsilon}$ and $B_{\epsilon}$:
\begin{align*}
\lambda(a) &:= \frac{Q_L(a)}{Q_L(a^*)},\\
\quad A_{\epsilon} &:= \{a\in \mathcal{A}| \lambda(a) \leq 1-2\epsilon\},\\
\quad B_{\epsilon} &:= \{a\in \mathcal{A}| \lambda(a) \geq 1-\epsilon\},
\end{align*}
where continuity of the map stems from $Q_L(a^*) > 0$ and continuity of $Q_L$ (Lemma~\ref{le:bellwl}). We will prove that the candidate set is $D=A_{\epsilon}^c$. In particular, we must prove that $A_{\epsilon}^c \subset U$ and that $(\pi_n(\cdot|s)\cdot\mu_A) (A_{\epsilon}) \rightarrow 0$. Using Lemma \ref{le:levelsets} (Appendix) on function $\lambda$, we can choose $\epsilon>0$ such that $A_{\epsilon}^c \subset U$, satisfying the first condition. We are left to prove that $(\pi_n(\cdot|s)\cdot\mu_A) (A_{\epsilon}) \rightarrow 0$.

Assume $A_{\epsilon} \neq \emptyset$ (otherwise the condition is proven): for all $a \in A_{\epsilon}$ and $b \in B_{\epsilon}$ it holds:
\begin{multline*}
\frac{Q_L(a)}{Q_L(b)} =
\frac{\frac{Q_L(a)}{Q_L(a^*)}}{\frac{Q_L(b)}{Q_L(a^*)}}
\leq
\frac{Q_L(a)}{Q_L(a^*)(1-\epsilon)}
\\
\leq
\frac{1-2\epsilon}{1-\epsilon}
=
1-\frac{\epsilon}{1-\epsilon}
=: \alpha_1 < 1.
\end{multline*}
For Lemma~\ref{le:bellwl} $Q_n$ converges uniformly to $Q_L$. Therefore we can fix $n_0>0$ such that $\|Q_n - Q_L\|_{\infty} < \epsilon'$ for all $n \geq n_0$, where we define
$\epsilon' := 0.1\times Q_L(a^*)(1-\epsilon)(1-\alpha_1)$. Now we can
proceed by bounding $Q_n$ ratio from above. For all $n \geq n_0$, $a \in A_{\epsilon}$ and $b \in B_{\epsilon}$:
\begin{align*}
\frac{Q_n(a)}{Q_n(b)} & \leq \frac{Q_L(a)}{Q_L(b)-\epsilon'}
\leq
\frac{Q_L(a)}{Q_L(a^*)(1-\epsilon)-\epsilon'}
\\
&=
\frac{Q_L(a)}{Q_L(a^*)(1-\epsilon)(1-0.1(1-\alpha_1) )} \\
& = \frac{\alpha_1}{(0.9 + 0.1\alpha_1)} =: \alpha < 1.
\end{align*}
Finally, we can bound the policy ratio. For all $n \geq n_0$, $a \in A_{\epsilon}$, $b \in B_{\epsilon}$:
$$
\frac{\pi_n(a|s)}{\pi_n(b|s)}
=
\frac{\pi_0(a|s)}{\pi_0(b|s)}
\prod_{i =0}^{n} \frac{Q_i(s,a)}{Q_i(s,b)}
\leq
\alpha^n c(a,b)
,
$$
where
$$
c(a,b) := \alpha^{-n_0}
\frac{\pi_0(a|s)}{\pi_0(b|s)}
\prod_{i=0}^{n_0} \frac{Q_i(s,a)}{Q_i(s,b)}.
$$
The function $c:A_{\epsilon}\times B_{\epsilon} \rightarrow \mathbb{R}^+$
is continuous as $\pi_0, Q_i$ are continuous (and denominators are non-zero due to
$\pi_0(b|s) > 0$ and $Q_i(s,b)>0$). Since $A_{\epsilon}\times B_{\epsilon}$ is a compact set, there exists $c_m$ such that $c\leq c_m$. Thus we have that for all $n>n_0$:
$$
\pi_n(a|s) \leq \alpha^n c_m \pi_n(b|s).
$$
Integrating with respect to $a$ over $A_{\epsilon}$ and then with respect to $b$
over $B_{\epsilon}$ (using reference measure $\mu_A$ in both cases) we obtain:
\begin{multline*}
(\pi_n(\cdot|s)\cdot\mu_A)(A_{\epsilon})
\times
(\mu_A B_{\epsilon})
\\
\leq \alpha^n c_m
(\pi_n(\cdot|s)\cdot\mu_A)(B_{\epsilon})
\times
(\mu_A A_{\epsilon}).
\end{multline*}
Rearranging terms, we have:
\begin{multline*}
(\pi_n(\cdot|s)\cdot\mu_A)(A_{\epsilon})
\\
\leq
\alpha^n
\left [
c_m
\frac{\mu_A A_{\epsilon}}{\mu_A B_{\epsilon} }
(\pi_n(\cdot|s)\cdot\mu_A) B_{\epsilon}
\right ]
\rightarrow 0, n \rightarrow \infty,
\end{multline*}
since the nominator in brackets is composed by finite measures of sets, thus finite numbers, while the denominator $\mu_A B_{\epsilon}>0$. Indeed, define the open set $C:= \{a\in \mathcal{A}| \lambda(a) > 1-\epsilon\} \subset B_{\epsilon}$. Then $\mu_A (B_{\epsilon}) \geq \mu_A (C) >0$ ($\mu_A$ is strictly positive). To conclude, we have proven that for arbitrarily small $\epsilon >0$, the term $(\pi_n(\cdot|s)\cdot\mu_A)(A_{\epsilon})$ tends to $0$, satisfying the condition for relative weak convergence of $\pi_n(\cdot|s)\cdot\mu_A \rightarrow^{w(M(s))} \pi_L(\cdot|s)$.
\end{proof}
\begin{lemma}
\label{le:Bseqprop}
Assume that, for each $s \in \mathcal{S}$, for each $\pi_L \in \Pi_L$, we have that  $\pi_n(\cdot|s)\cdot\mu_A \rightarrow^{w(M(s))} \pi_L(\cdot|s) (\in \mathcal{P}(M(s)))$. Then this holds:
\begin{equation}
V_L(s) = \int_{\mathcal{A}} Q_L(s,a) \de \pi_L(a|s).
\end{equation}
\end{lemma}
\begin{proof}
Fix $s \in \mathcal{S}$ and $\pi_L \in \Pi_L$. We aim to show $V_L(s) - \int_{\mathcal{A}} Q_L(s,a) \de \pi_L(a|s) = 0$.
Since $V_n(s) - \int_{\mathcal{A}} Q_n(s,a) \pi_n(a|s) \de \mu_A(a) = 0$, we have:
\begin{multline*}
\Bigl| V_L(s) - \int_{\mathcal{A}} Q_L(s,a) \de \pi_L(a|s) \Bigr|\\
\begin{aligned}
& = \Bigl| V_L(s) - V_n(s) - \int_{\mathcal{A}} Q_L(s,a) \de \pi_L(a|s)
\\&\quad+ \int_{\mathcal{A}} Q_n(s,a) \pi_n(a|s) \de \mu_A(a) \Bigr|
\\
&\leq
\Bigl| V_L(s) - V_n(s) \Bigr| + \Bigl| \int_{\mathcal{A}} Q_L(s,a) \de \pi_L(a|s)
\\&\quad- \int_{\mathcal{A}} Q_n(s,a) \pi_n(a|s) \de \mu_A(a) \Bigr|.
\end{aligned}
\end{multline*}
The first part can be made arbitrarily small due to $V_n(s) \rightarrow V_L(s)$. For the second part:
\begin{align*}
&\Bigl| \int_{\mathcal{A}} Q_L(s,a) \de \pi_L(a|s)  - \int_{\mathcal{A}} Q_n(s,a) \pi_n(a|s) \de \mu_A(a) \Bigr|\\
&= \Bigl| \int_{\mathcal{A}} Q_L(s,a)\dif\pi_L(a|s) - \int_{\mathcal{A}} Q_L(s,a) \pi_n(a|s) \dif\mu_A(a)
\\&\quad + \int_{\mathcal{A}} Q_L(s,a) \pi_n(a|s) \dif\mu_A(a)
\\&\quad - \int_{\mathcal{A}} Q_n(s,a) \pi_n(a|s) \de \mu_A(a) \Bigr|
\\&\leq
\Bigl|\int_{\mathcal{A}} Q_L(s,a)\dif\pi_L(a|s)
-\int_{\mathcal{A}} Q_L(s,a)\pi_n(a|s)\dif\mu_A(a)\Bigr|
\\&\quad + \int_{\mathcal{A}} | Q_L(s,a)- Q_n(s,a)| \pi_n(a|s) \dif\mu_A(a),
\end{align*}
where the first term tends to zero since $\pi_n(\cdot|s)\cdot\mu_A \rightarrow^{w(M(s))} \pi_L(\cdot|s)$ and $Q_L$ is continuous and constant on $M(s)$, satisfying the conditions of the adapted Portmanteau Lemma~\citep{billingsley2013convergence} (see Appendix~\ref{le:RPort}).
The second term can be arbitrarily small since Lemma~\ref{le:bellwl} ensures uniform convergence of $Q_n$ to $Q_L$.
\end{proof}
\begin{restatable}[]{thr}{ct}
\label{th:CT}
Let $\pi_n$ be a sequence generated by $\pi_n := B(\pi_{n-1})$. Let $\pi_0$ be such that $\forall s \in \mathcal{S}$, $\forall a \in \mathcal{A}$ $\pi_0(a|s)>0$ and continuous in actions. Then $\forall s\in \mathcal{S} \quad \pi_n(\cdot|s)\cdot\mu_A \rightarrow^{w(M(s))} \pi_L(\cdot|s)$, where $\pi_L  \in \Pi_L$ is an optimal policy for the MDP. Moreover, $\lim_{n \rightarrow \infty} V_n = V_L$, $\lim_{n \rightarrow \infty} Q_n = Q_L$ are the optimal state and action value functions.
\end{restatable}
\begin{proof}
Fix $\pi_L \in \Pi_L$ (we have already shown that $\Pi_L \neq \emptyset$). Due to Lemma~\ref{le:greedylim}, we know that for all $s\in \mathcal{S}$, $\pi_L(\cdot|s)$ is the relative weak limit $\pi_n(\cdot|s)\cdot\mu_A \rightarrow^{w(M(s))} \pi_L(\cdot|s)$  and further we know that $\pi_L$ is greedy on $Q_L(s,a)$ (from definition of $\Pi_L$). Moreover, thanks to Lemmas~\ref{le:Bseqprop} and~\ref{le:bellwl}, $V_L(s)$ and $Q_L(s,a)$ are the state and action value functions of $\pi_L$ because they are fixed points of the Bellman operator. Since $\pi_L(\cdot|s) \in \mathcal{P}(\argmax_a Q_L(s,a))$, $V_L(s)$ and $Q_L(s,a)$ are also the unique fixed points of Bellman's optimality operator, hence $V_L$, $Q_L$ are optimal value functions and $\pi_L$ is an optimal policy.
\end{proof}

This result has several implications. First, it provides a solid theoretical ground for both previous and future works that are based on RWR \cite{dayan1997using,peters2007reinforcement,peng2019advantage} and lends us some additional understanding regarding the properties of similar algorithms (e.g., \cite{abdolmaleki2018maximum}). It should also be stressed that the results presented herein are for compact state and action spaces: traits of some key domains such as robotic control. In addition to the above, one should also note that the upper bound on $(\pi_n(\cdot|s)\cdot\mu_A)(A_{\epsilon})$ constructed in lemma \ref{le:greedylim} can be used to study convergence orders and convergence rates of RWR. The following corollary, for example, proves R-linear convergence for the special case of finite state and action spaces:

\begin{restatable}{coroll}{rlinearconv}
Under the assumptions of lemma \ref{le:greedylim}, if $\mathcal{S}$ and $\mathcal{A}$ are finite, then $\| V^*- V_n \|_{\infty} = O(\alpha_m^n)$,
where $0 \leq \alpha_m < 1$,
$\alpha_m := \frac{2\lambda_m}{0.9+1.1\lambda_m}$,
and
$\lambda_m := \max_{s\in \mathcal{S}} \max_{a\in \mathcal{A} \setminus M(s)}
\frac{Q^*(s,a)}{V^*(s)}$.
\label{le:ratescol}
\end{restatable}

A proof of the above is included in the appendix \ref{se:crates}. We observe that in the finite case, $\| V^*- V_n \|_{\infty}$ converges to $0$ R-linearly (i.e., $\| V^*- V_n \|_{\infty}$ is bounded by a Q-linearly converging sequence $\alpha_m^n$). We provide an example of a finite MDP in lemma \ref{le:fexample} which exhibits linear convergence rate, showing that the upper bound from the corollary \ref{le:ratescol} is asymptotically tight in regards to the convergence order. Therefore it is not possible to achieve an order of convergence better than linear. Furthermore, the example in lemma \ref{le:cexample} shows that, in the continuous case, the convergence order could be sub-linear. Appendix~\ref{se:relevance} discusses the motivation of our approach.

\section{Demonstration of RWR Convergence}
\label{sec:exp}

To illustrate that the update scheme of Theorem~\ref{th:multbyQ} converges to the optimal policy, we test it on a simple environment that meets the assumptions of the Theorem. In particular, we ensure that rewards are positive and that there is no function approximations for value functions and policies. In order to meet these criteria, we use the modified four-room gridworld domain~\citep{sutton1999between} shown on the left of Figure~\ref{fig:results}. Here the agent starts in the upper left corner and must navigate to the bottom right corner (i.e., the goal state). In non-goal states actions are restricted to moving one square at each step in any of the four cardinal directions. If the agent tries to move into a square containing a wall, it will remain in place. In the goal state, all actions lead to the agent remaining in place. The agent receives a reward of $1$ when transitioning from a non-goal state to the goal state and a reward of $0.001$ otherwise. The discount-rate is $0.9$ at each step. At each iteration, we use Bellman's updates to obtain a reliable estimate of $Q_n$ and $V_n$, before updating $\pi_n$ using the operator in Theorem~\ref{th:multbyQ}.

\begin{figure}
    \centering
    \includegraphics[width=0.56\columnwidth]{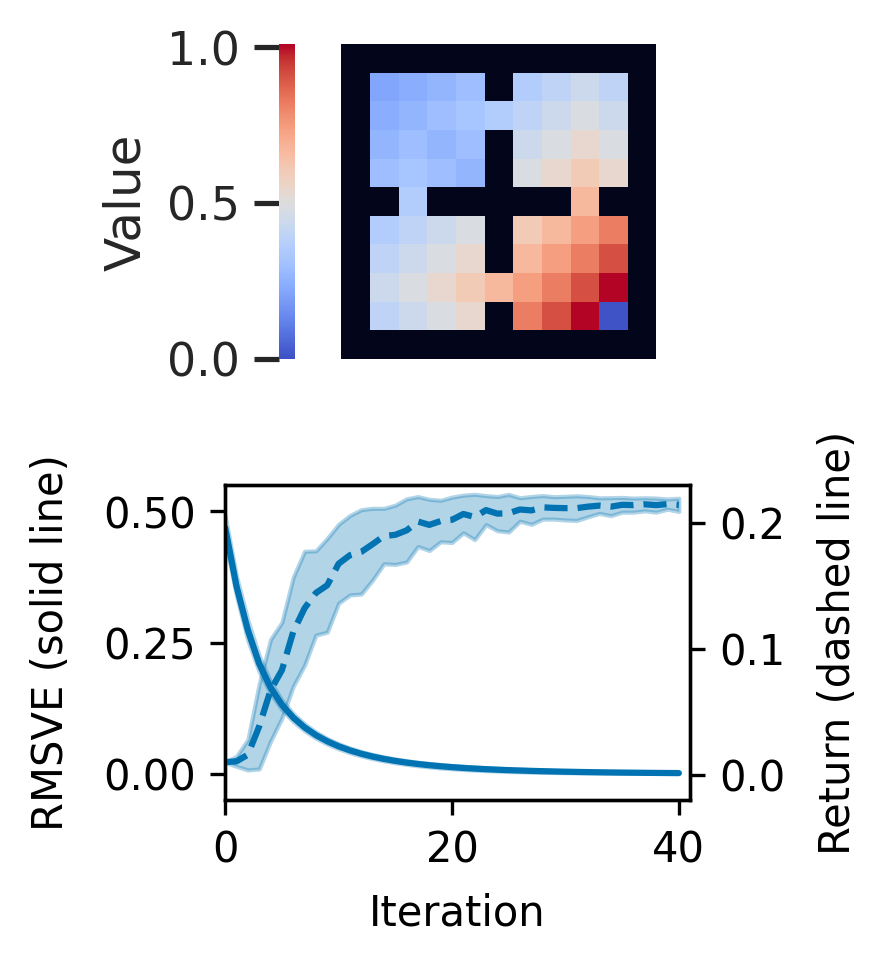}
    \caption{(Top) the value of states under the optimal policy in the four-room gridworld domain. (Bottom) the root-mean-squared value error of reward-weighted regression in the four-room gridworld domain---compared to the optimal policy---and the return obtained by running the learned policy of reward-weighted regression. All lines are averages of 100 runs under different uniform random initial policies. Shading shows standard deviation.}
    \label{fig:results}
\end{figure}

The bottom left of Figure~\ref{fig:results} shows the root-mean-squared value error (RMSVE) of the learned policy at each iteration as compared to the optimal policy, while the bottom right shows the return obtained by the learned policy at each iteration. Smooth convergence can be observed under reward-weighted regression. The source code for this experiment is available at \url{https://github.com/dylanashley/reward-weighted-regression}.

\section{Related Work}
\label{sec:related_work}

The principle behind expectation-maximization was first applied to artificial neural networks by \citet{von1973self}. The reward-weighted regression (RWR) algorithm, though, originated in the work of \citet{peters2007reinforcement} which sought to bring earlier work of \citet{dayan1997using} to the domain of operational space control and reinforcement learning. However, \citet{peters2007reinforcement} only considered the immediate-reward reinforcement learning (RL) setting. This was later extended to the episodic setting separately by \citet{wierstra2008episodic} and then by \citet{kober2011policy}. \citet{wierstra2008episodic} went even further and also extended RWR to partially observable Markov decision processes, and \citet{kober2011policy} applied it to motor learning in robotics. Separately, \citet{wierstra2008fitness} extended RWR to perform fitness maximization for evolutionary methods. \citet{hachiya2009efficient,hachiya2011reward} later found a way of reusing old samples to improve RWR's sample complexity. Much later, \citet{peng2019advantage} modified RWR to produce an algorithm for off-policy RL, using deep neural networks as function approximators.

Other methods based on principles similar to RWR have been proposed. \citet{neumann2009fitted}, for example, proposed a more efficient version of the well-known fitted Q-iteration algorithm~\citep{riedmiller2005neural,ernst2005tree,antos2007fitted} by using what they refer to as \textit{advantaged-weighted regression}---which itself is based on the RWR principle. \citet{ueno2012weighted} later proposed \textit{weighted likelihood policy search} and showed that their method both has guaranteed monotonic increases in the expected reward. \citet{osa2018hierarchical} subsequently proposed a hierarchical RL method called \textit{hierarchical policy search via return-weighted density estimation} and showed that it is closely related to the episodic version of RWR by \cite{kober2011policy}.

Notably, all of the aforementioned works, as well as a number of other proposed similar RL methods (e.g., \citet{peters2010relative}, \citet{neumann2011variational}, \citet{abdolmaleki2018maximum}, \citet{abdolmaleki2018relative}), are based on the expectation-maximization framework of \citet{dempster1977maximum} and are thus known to have monotonic improvements of the policy in the RL setting under certain conditions. However, it has remained an open question under which conditions convergence to the optimal is guaranteed.

\section{Conclusion and Future Work}
\label{sec:conclusion}

We provided the first global convergence proof for Reward-Weighted Regression (RWR) in absence of reward transformation and function approximation. The convergence achieved is linear when using finite state and action spaces and can be sub-linear in the continuous case. We also highlighted problems that may arise under nonlinear reward transformations, potentially resulting in changes to the optimal policy. In real-world problems, access to true value functions may be unrealistic. Future work will study RWR's convergence under function approximation.
In such a case, the best scenario that one can expect is to achieve convergence to a local optimum. One possible approach is to follow a procedure similar to standard policy gradients \cite{Sutton1999} and derive a class of value function approximators that is compatible with the RWR objective. It might be possible then to prove local convergence under value function approximation using stochastic approximation techniques \cite{borkar2009stochastic, gtd, fastgd}. This would require casting the value function and policy updates in a system of equations and studying the convergence of the corresponding ODE under specific assumptions.
Our RWR is on-policy, using only recent data to update the current policy. Future work will also study convergence in challenging off-policy settings (using all past data), which require corrections of the mismatch between state-distributions, typically through a mechanism like Importance Sampling.

\section*{Acknowledgements}
We would like to thank Sjoerd van Steenkiste and Franti\v{s}ek \v{Z}\'{a}k for their insightful comments. This work was supported by the European Research Council (ERC, Advanced Grant Number 742870), the Swiss National Supercomputing Centre (CSCS, Project s1090), and by the Swiss National Science Foundation (Grant Number  200021\_192356, Project NEUSYM). We also thank both the NVIDIA Corporation for donating a DGX-1 as part of the Pioneers of AI Research Award and IBM for donating a Minsky machine.

\bibliography{main}

\begin{thebibliography}{38}
\providecommand{\natexlab}[1]{#1}

\bibitem[{Abdolmaleki et~al.(2018{\natexlab{a}})Abdolmaleki, Springenberg,
  Degrave, Bohez, Tassa, Belov, Heess, and
  Riedmiller}]{abdolmaleki2018relative}
Abdolmaleki, A.; Springenberg, J.~T.; Degrave, J.; Bohez, S.; Tassa, Y.; Belov,
  D.; Heess, N.; and Riedmiller, M. 2018{\natexlab{a}}.
\newblock Relative Entropy Regularized Policy Iteration.
\newblock arXiv:1812.02256.

\bibitem[{Abdolmaleki et~al.(2018{\natexlab{b}})Abdolmaleki, Springenberg,
  Tassa, Munos, Heess, and Riedmiller}]{abdolmaleki2018maximum}
Abdolmaleki, A.; Springenberg, J.~T.; Tassa, Y.; Munos, R.; Heess, N.; and
  Riedmiller, M.~A. 2018{\natexlab{b}}.
\newblock Maximum a Posteriori Policy Optimisation.
\newblock In \emph{6th International Conference on Learning Representations,
  {ICLR} 2018, Vancouver, BC, Canada, April 30 - May 3, 2018, Conference Track
  Proceedings}. OpenReview.net.

\bibitem[{Antos, Munos, and Szepesv{\'{a}}ri(2007)}]{antos2007fitted}
Antos, A.; Munos, R.; and Szepesv{\'{a}}ri, C. 2007.
\newblock Fitted Q-iteration in continuous action-space MDPs.
\newblock In Platt, J.~C.; Koller, D.; Singer, Y.; and Roweis, S.~T., eds.,
  \emph{Advances in Neural Information Processing Systems 20, Proceedings of
  the Twenty-First Annual Conference on Neural Information Processing Systems,
  Vancouver, British Columbia, Canada, December 3-6, 2007}, 9--16. Curran
  Associates, Inc.

\bibitem[{Bellemare et~al.(2020)Bellemare, Candido, Castro, Gong, Machado,
  Moitra, Ponda, and Wang}]{bellemare2020autonomous}
Bellemare, M.~G.; Candido, S.; Castro, P.~S.; Gong, J.; Machado, M.~C.; Moitra,
  S.; Ponda, S.~S.; and Wang, Z. 2020.
\newblock Autonomous navigation of stratospheric balloons using reinforcement
  learning.
\newblock \emph{Nature}, 588(7836): 77--82.

\bibitem[{Billingsley(2013)}]{billingsley2013convergence}
Billingsley, P. 2013.
\newblock \emph{Convergence of Probability Measures}.
\newblock John Wiley \& Sons.

\bibitem[{Borkar(2008)}]{borkar2009stochastic}
Borkar, V.~S. 2008.
\newblock \emph{Stochastic Approximation: A Dynamical Systems Viewpoint},
  volume~48.
\newblock Hindustan Book Agency.

\bibitem[{Dayan and Hinton(1997)}]{dayan1997using}
Dayan, P.; and Hinton, G.~E. 1997.
\newblock Using Expectation-Maximization for Reinforcement Learning.
\newblock \emph{Neural Comput.}, 9(2): 271--278.

\bibitem[{Dempster, Laird, and Rubin(1977)}]{dempster1977maximum}
Dempster, A.~P.; Laird, N.~M.; and Rubin, D.~B. 1977.
\newblock Maximum likelihood from incomplete data via the EM algorithm.
\newblock \emph{Journal of the Royal Statistical Society: Series B
  (Methodological)}, 39(1): 1--22.

\bibitem[{Ernst, Geurts, and Wehenkel(2005)}]{ernst2005tree}
Ernst, D.; Geurts, P.; and Wehenkel, L. 2005.
\newblock Tree-Based Batch Mode Reinforcement Learning.
\newblock \emph{J. Mach. Learn. Res.}, 6: 503--556.

\bibitem[{Hachiya, Peters, and Sugiyama(2009)}]{hachiya2009efficient}
Hachiya, H.; Peters, J.; and Sugiyama, M. 2009.
\newblock Efficient Sample Reuse in EM-Based Policy Search.
\newblock In Buntine, W.~L.; Grobelnik, M.; Mladenic, D.; and Shawe{-}Taylor,
  J., eds., \emph{Machine Learning and Knowledge Discovery in Databases,
  European Conference, {ECML} {PKDD} 2009, Bled, Slovenia, September 7-11,
  2009, Proceedings, Part {I}}, volume 5781 of \emph{Lecture Notes in Computer
  Science}, 469--484. Springer.

\bibitem[{Hachiya, Peters, and Sugiyama(2011)}]{hachiya2011reward}
Hachiya, H.; Peters, J.; and Sugiyama, M. 2011.
\newblock Reward-Weighted Regression with Sample Reuse for Direct Policy Search
  in Reinforcement Learning.
\newblock \emph{Neural Comput.}, 23(11): 2798--2832.

\bibitem[{Kober and Peters(2011)}]{kober2011policy}
Kober, J.; and Peters, J. 2011.
\newblock Policy search for motor primitives in robotics.
\newblock \emph{Mach. Learn.}, 84(1-2): 171--203.

\bibitem[{Munkres(2000)}]{munkres2000topology}
Munkres, J. 2000.
\newblock \emph{Topology}.
\newblock Prentice Hall, Incorporated.

\bibitem[{Neumann(2011)}]{neumann2011variational}
Neumann, G. 2011.
\newblock Variational Inference for Policy Search in changing situations.
\newblock In Getoor, L.; and Scheffer, T., eds., \emph{Proceedings of the 28th
  International Conference on Machine Learning, {ICML} 2011, Bellevue,
  Washington, USA, June 28 - July 2, 2011}, 817--824. Omnipress.

\bibitem[{Neumann and Peters(2008)}]{neumann2009fitted}
Neumann, G.; and Peters, J. 2008.
\newblock Fitted Q-iteration by Advantage Weighted Regression.
\newblock In Koller, D.; Schuurmans, D.; Bengio, Y.; and Bottou, L., eds.,
  \emph{Advances in Neural Information Processing Systems 21, Proceedings of
  the Twenty-Second Annual Conference on Neural Information Processing Systems,
  Vancouver, British Columbia, Canada, December 8-11, 2008}, 1177--1184. Curran
  Associates, Inc.

\bibitem[{Osa and Sugiyama(2018)}]{osa2018hierarchical}
Osa, T.; and Sugiyama, M. 2018.
\newblock Hierarchical Policy Search via Return-Weighted Density Estimation.
\newblock In McIlraith, S.~A.; and Weinberger, K.~Q., eds., \emph{Proceedings
  of the Thirty-Second {AAAI} Conference on Artificial Intelligence, (AAAI-18),
  the 30th innovative Applications of Artificial Intelligence (IAAI-18), and
  the 8th {AAAI} Symposium on Educational Advances in Artificial Intelligence
  (EAAI-18), New Orleans, Louisiana, USA, February 2-7, 2018}, 3860--3867.
  {AAAI} Press.

\bibitem[{Peng et~al.(2019)Peng, Kumar, Zhang, and Levine}]{peng2019advantage}
Peng, X.~B.; Kumar, A.; Zhang, G.; and Levine, S. 2019.
\newblock Advantage-Weighted Regression: Simple and Scalable Off-Policy
  Reinforcement Learning.
\newblock arXiv:1910.00177.

\bibitem[{Peters, M{\"{u}}lling, and Altun(2010)}]{peters2010relative}
Peters, J.; M{\"{u}}lling, K.; and Altun, Y. 2010.
\newblock Relative Entropy Policy Search.
\newblock In Fox, M.; and Poole, D., eds., \emph{Proceedings of the
  Twenty-Fourth {AAAI} Conference on Artificial Intelligence, {AAAI} 2010,
  Atlanta, Georgia, USA, July 11-15, 2010}. {AAAI} Press.

\bibitem[{Peters and Schaal(2007)}]{peters2007reinforcement}
Peters, J.; and Schaal, S. 2007.
\newblock Reinforcement Learning by Reward-Weighted Regression for Operational
  Space Control.
\newblock In Ghahramani, Z., ed., \emph{Machine Learning, Proceedings of the
  Twenty-Fourth International Conference {(ICML} 2007), Corvallis, Oregon, USA,
  June 20-24, 2007}, volume 227 of \emph{{ACM} International Conference
  Proceeding Series}, 745--750. {ACM}.

\bibitem[{Peters and Schaal(2008)}]{peters2008learning}
Peters, J.; and Schaal, S. 2008.
\newblock Learning to Control in Operational Space.
\newblock \emph{The International Journal of Robotics Research}, 27(2):
  197--212.

\bibitem[{Plappert et~al.(2018)Plappert, Andrychowicz, Ray, McGrew, Baker,
  Powell, Schneider, Tobin, Chociej, Welinder, Kumar, and
  Zaremba}]{plappert2018multi}
Plappert, M.; Andrychowicz, M.; Ray, A.; McGrew, B.; Baker, B.; Powell, G.;
  Schneider, J.; Tobin, J.; Chociej, M.; Welinder, P.; Kumar, V.; and Zaremba,
  W. 2018.
\newblock Multi-Goal Reinforcement Learning: Challenging Robotics Environments
  and Request for Research.
\newblock arXiv:1802.09464.

\bibitem[{Pollard(2001)}]{pollard2001user}
Pollard, D. 2001.
\newblock \emph{A User's Guide to Measure Theoretic Probability}.
\newblock Cambridge Series in Statistical and Probabilistic Mathematics.
  Cambridge University Press.

\bibitem[{Puterman(2014)}]{puterman2014markov}
Puterman, M.~L. 2014.
\newblock \emph{Markov Decision Processes: Discrete Stochastic Dynamic
  Programming}.
\newblock John Wiley \& Sons.

\bibitem[{Riedmiller(2005)}]{riedmiller2005neural}
Riedmiller, M.~A. 2005.
\newblock Neural Fitted {Q} Iteration - First Experiences with a Data Efficient
  Neural Reinforcement Learning Method.
\newblock In Gama, J.; Camacho, R.; Brazdil, P.; Jorge, A.; and Torgo, L.,
  eds., \emph{Machine Learning: {ECML} 2005, 16th European Conference on
  Machine Learning, Porto, Portugal, October 3-7, 2005, Proceedings}, volume
  3720 of \emph{Lecture Notes in Computer Science}, 317--328. Springer.

\bibitem[{Rudin(1976)}]{rudin1976principles}
Rudin, W. 1976.
\newblock \emph{Principles of Mathematical Analysis}.
\newblock McGraw-hill New York, 3d ed. edition.

\bibitem[{Severini(1910)}]{severini}
Severini, C. 1910.
\newblock Sulle successioni di funzioni ortogonali [On Sequences of Orthogonal
  Functions].
\newblock \emph{Atti dell'Accademia Gioenia, serie 5a (in Italian), 3 (5):
  Memoria XIII, 1-7, JFM 41.0475.04. Published by the Accademia Gioenia in
  Catania}.

\bibitem[{Silver et~al.(2016)Silver, Huang, Maddison, Guez, Sifre, van~den
  Driessche, Schrittwieser, Antonoglou, Panneershelvam, Lanctot, Dieleman,
  Grewe, Nham, Kalchbrenner, Sutskever, Lillicrap, Leach, Kavukcuoglu, Graepel,
  and Hassabis}]{silver2016mastering}
Silver, D.; Huang, A.; Maddison, C.~J.; Guez, A.; Sifre, L.; van~den Driessche,
  G.; Schrittwieser, J.; Antonoglou, I.; Panneershelvam, V.; Lanctot, M.;
  Dieleman, S.; Grewe, D.; Nham, J.; Kalchbrenner, N.; Sutskever, I.;
  Lillicrap, T.~P.; Leach, M.; Kavukcuoglu, K.; Graepel, T.; and Hassabis, D.
  2016.
\newblock Mastering the game of Go with deep neural networks and tree search.
\newblock \emph{Nat.}, 529(7587): 484--489.

\bibitem[{Stratonovich(1960)}]{stratonovich1960}
Stratonovich, R. 1960.
\newblock Conditional {Markov} processes.
\newblock \emph{Theory of Probability And Its Applications}, 5(2): 156--178.

\bibitem[{Sutton and Barto(2018)}]{Sutton:2018:RLI:3312046}
Sutton, R.~S.; and Barto, A.~G. 2018.
\newblock \emph{Reinforcement Learning: An Introduction}.
\newblock USA: A Bradford Book.
\newblock ISBN 0262039249, 9780262039246.

\bibitem[{Sutton et~al.(2009)Sutton, Maei, Precup, Bhatnagar, Silver,
  Szepesv\'{a}ri, and Wiewiora}]{fastgd}
Sutton, R.~S.; Maei, H.~R.; Precup, D.; Bhatnagar, S.; Silver, D.;
  Szepesv\'{a}ri, C.; and Wiewiora, E. 2009.
\newblock Fast Gradient-Descent Methods for Temporal-Difference Learning with
  Linear Function Approximation.
\newblock In \emph{Proceedings of the 26th Annual International Conference on
  Machine Learning}, ICML ’09, 993–1000. New York, NY, USA: Association for
  Computing Machinery.
\newblock ISBN 9781605585161.

\bibitem[{Sutton, Maei, and Szepesv{\'a}ri(2009)}]{gtd}
Sutton, R.~S.; Maei, H.~R.; and Szepesv{\'a}ri, C. 2009.
\newblock A convergent $ o (n) $ temporal-difference algorithm for off-policy
  learning with linear function approximation.
\newblock In \emph{Advances in neural information processing systems},
  1609--1616.

\bibitem[{Sutton et~al.(1999)Sutton, McAllester, Singh, and
  Mansour}]{Sutton1999}
Sutton, R.~S.; McAllester, D.; Singh, S.; and Mansour, Y. 1999.
\newblock Policy Gradient Methods for Reinforcement Learning with Function
  Approximation.
\newblock In \emph{Proceedings of the 12th International Conference on Neural
  Information Processing Systems}, NIPS'99, 1057--1063. Cambridge, MA, USA: MIT
  Press.

\bibitem[{Sutton, Precup, and Singh(1999)}]{sutton1999between}
Sutton, R.~S.; Precup, D.; and Singh, S.~P. 1999.
\newblock Between MDPs and Semi-MDPs: {A} Framework for Temporal Abstraction in
  Reinforcement Learning.
\newblock \emph{Artif. Intell.}, 112(1-2): 181--211.

\bibitem[{Ueno et~al.(2012)Ueno, Hayashi, Washio, and
  Kawahara}]{ueno2012weighted}
Ueno, T.; Hayashi, K.; Washio, T.; and Kawahara, Y. 2012.
\newblock Weighted Likelihood Policy Search with Model Selection.
\newblock In Bartlett, P.~L.; Pereira, F. C.~N.; Burges, C. J.~C.; Bottou, L.;
  and Weinberger, K.~Q., eds., \emph{Advances in Neural Information Processing
  Systems 25: 26th Annual Conference on Neural Information Processing Systems
  2012. Proceedings of a meeting held December 3-6, 2012, Lake Tahoe, Nevada,
  United States}, 2366--2374.

\bibitem[{Von~der Malsburg(1973)}]{von1973self}
Von~der Malsburg, C. 1973.
\newblock Self-organization of orientation sensitive cells in the striate
  cortex.
\newblock \emph{Kybernetik}, 14(2): 85--100.

\bibitem[{Wierstra et~al.(2008{\natexlab{a}})Wierstra, Schaul, Peters, and
  Schmidhuber}]{wierstra2008episodic}
Wierstra, D.; Schaul, T.; Peters, J.; and Schmidhuber, J. 2008{\natexlab{a}}.
\newblock Episodic Reinforcement Learning by Logistic Reward-Weighted
  Regression.
\newblock In Kurkov{\'{a}}, V.; Neruda, R.; and Koutn{\'{\i}}k, J., eds.,
  \emph{Artificial Neural Networks - {ICANN} 2008 , 18th International
  Conference, Prague, Czech Republic, September 3-6, 2008, Proceedings, Part
  {I}}, volume 5163 of \emph{Lecture Notes in Computer Science}, 407--416.
  Springer.

\bibitem[{Wierstra et~al.(2008{\natexlab{b}})Wierstra, Schaul, Peters, and
  Schmidhuber}]{wierstra2008fitness}
Wierstra, D.; Schaul, T.; Peters, J.; and Schmidhuber, J. 2008{\natexlab{b}}.
\newblock Fitness Expectation Maximization.
\newblock In Rudolph, G.; Jansen, T.; Lucas, S.~M.; Poloni, C.; and Beume, N.,
  eds., \emph{Parallel Problem Solving from Nature - {PPSN} X, 10th
  International Conference Dortmund, Germany, September 13-17, 2008,
  Proceedings}, volume 5199 of \emph{Lecture Notes in Computer Science},
  337--346. Springer.

\bibitem[{Wu(1983)}]{wu1983convergence}
Wu, C.~J. 1983.
\newblock On the Convergence Properties of the EM Algorithm.
\newblock \emph{The Annals of statistics}, 11(1): 95--103.

\end{thebibliography}

 \clearpage

 \appendix

 \section{Counterexample}
 \label{ap:CEsection}
 Consider the simple two-armed bandit shown in Figure~\ref{fig:counterexample} with actions $a_0$ and $a_1$, and with \(P(r = 1 | a_0) = 1\), \(P(r = 0 | a_1) = \sfrac{2}{3}\), and \(P(r = 2 | a_1) = \sfrac{1}{3}\). Note that \(q(a_0) = 1 > q(a_1) = \sfrac{2}{3}\). Thus the optimal policy always takes action $a_0$. Now, after applying the transformation $u(r) = e^{\log(3)\,r} = 3^r$, we get \(P(u(r) = 3 | a_0) = 1\), \(P(u(r) = 1 | a_1) = \sfrac{2}{3}\), and \(P(u(r) = 9 | a_1) = \sfrac{1}{3}\). Hence, under transformation $u$, we have \(q(a_0) = 3 < q(a_1) = \sfrac{11}{3}\). So the optimal policy under the transformed rewards always takes action $a_1$, which is sub-optimal, given the original problem.

 \begin{figure*}[t]
     \centering
     \includegraphics[width=1.5\columnwidth]{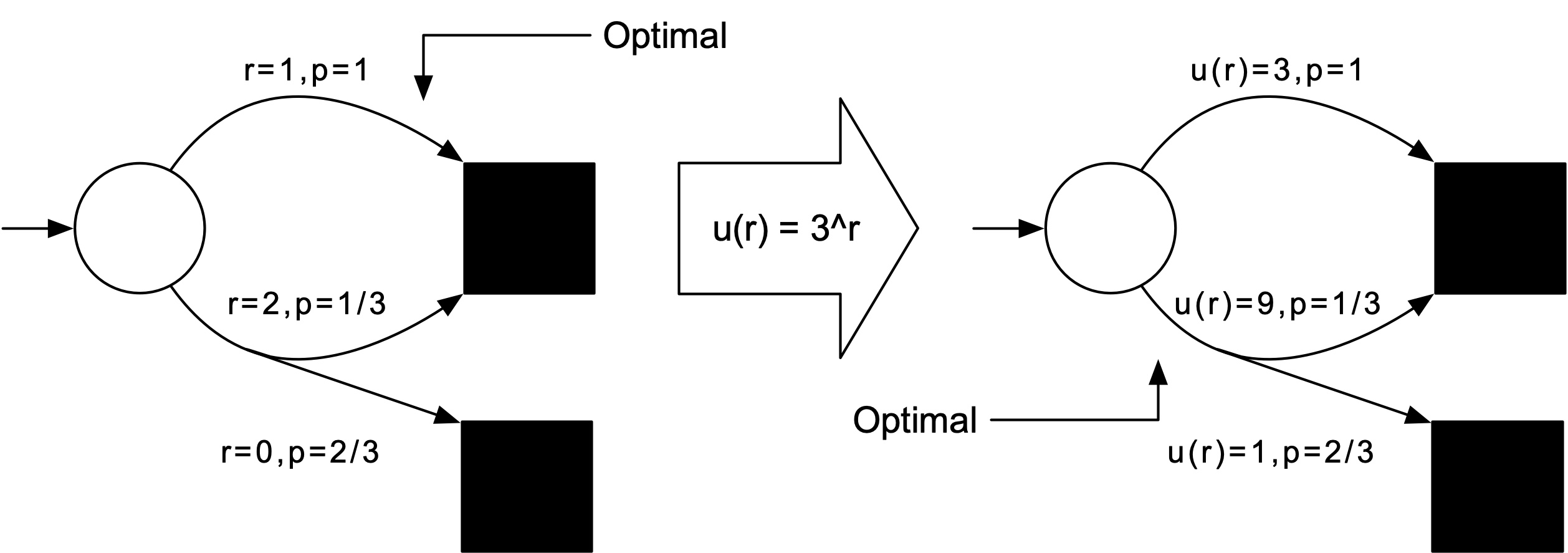}
     \caption{Counterexample demonstrating how applying a naive transformation to the reward function of an MDP may change the optimal policy.}
     \label{fig:counterexample}
 \end{figure*}

 \section{Generalization to zero reward and undiscounted setting}
 \label{ap:MITsection}
 In this section, we study the generalization of the main results of the paper to the case of MDPs with absorbing states (when $\gamma = 1$)
 and when rewards are not strictly positive.
 Full treatment of this topic is outside the scope of this paper and we restrict ourselves here to a discussion of the main problems and possible assumptions.

 \paragraph{Motivation}
 When rewards are not strictly positive (e.g., are allowed to be zero), the requirement for shifting rewards to a positive range can be very inconvenient. In particular, there can be a convergence slowdown (see proof of lemma \ref{le:greedylim}: if we shift $Q_L$ by positive constant, the bound $\alpha_1$ and consequently $\alpha$ gets closer to 1, leading to potentially slower convergence). This was also observed in previous work~\citep{dayan1997using}.

 Moreover, assuming $\gamma < 1$ prevents us from studying many useful cases, e.g. the simple undiscounted fixed horizon case. However, it is possible to model this setting using an MDP with absorbing states. We define this new case as follows:

 \begin{definition}
 \label{de:unibound}
 (MDP with absorbing states, uniform boundedness)
 Let us have a MDP
 $\mathcal{M}= (\mathcal{S},\mathcal{A},p_T,R,\gamma,\mu_0)$ with $\gamma = 1$. The state $s \in \mathcal{S}$ is \textbf{absorbing} if and only if,
 for all actions $a \in \mathcal{A}$ it holds:
 $(p_T(.|s,a)\cdot\mu_S)(\{s\}) = 1$.
 Thus after the MDP enters an absorbing state, it stays there with probability 1.
 Let $S_A \subset \mathcal{S}$ measurable be \textbf{the set of all absorbing states}.
 All other states are called \textbf{transient}.
 We denote by $S_T := \mathcal{S} \setminus S_A$ \textbf{the set of all transient states} (which is also measurable from measurability of $\mathcal{S}$).
 We will consider just MDPs here where there is zero reward from absorbing states (i.e. $\forall s\in S_A, \forall a \in \mathcal{A}, R(s,a)=0$).

 Denote by $\hat{P}_{\pi} : L_{\infty}(S_T) \rightarrow L_{\infty}(S_T)$
 the operator arising from the restriction of $\mathcal{M}$ transition kernel to transient states, for a fixed policy $\pi$:
 \begin{multline*}
 \forall s \in S_T, V \in L_{\infty}(S_T),
 \\
 [\hat{P}_{\pi}(V)](s) := \int_{S_T} V(s') p_{\pi} (s'|s) \dif s',
 \end{multline*}
 where\footnote{It is more flexible to allow here for general policy measures, not just measures dominated by the reference $\mu_A$.}
 $$
 p_{\pi} (s'|s) := \int_{\mathcal{A}} p_{T}(s'|s,a) \dif \pi(a|s)
 .
 $$
 The transition kernel of and MDP $\mathcal{M}$ with absorbing states is called \textbf{uniformly bounded}
 if and only if there exists $k_u \in \mathbb{N}$ and $\alpha < 1$ non-negative such that
 for every sequence of policies $\pi_0,\ldots \pi_{k_u}$ of length $k_u$, it holds:
 $$
 \| \hat{P}_{\pi_0}\circ\hat{P}_{\pi_1}\circ\ldots\hat{P}_{\pi_{k_u}} \|_{\infty} < \alpha
 .
 $$
 \end{definition}

 The definition above requires a few comments. One can see that we were strongly motivated by the classical "discrete" case, which we always
 want to include as a sub-case. Finite ("discrete") state space case usually comes with the condition that starting from all $s \in \mathcal{S}$
 we eventually end up in an absorbing state. In the MDP setting we also have to specify a policy, or sequence of policies used $(\pi_{t'})$ (i.e.
 $(p_{s_t|s_0,(\pi_{t'})}(.|s)\cdot\mu_S)(S_A) \rightarrow 1$).
 Unless we are in the finite state and action spaces setting, it is very difficult to establish boundedness of value functions, which we need in the proofs. This motivates the introduction of stronger assumptions like uniform boundedness above.
 Moreover, the uniform boundedness condition allows us to proceed in a similar way as in discounted case.
 The only difference is that the Bellman operator and the Bellman optimality operator
 are generally not contractions anymore: only their restricted versions (to $S_T$) when composed $k_u$ times are contractions. However, the restricted versions of original operators
 inherit all useful "limit" and "fixed point" properties from their $k_u$ composition,
 so the proofs in main paper can be adapted to them.

 Regarding monotonicity, when defining the expression of a new policy in theorem \ref{th:multbyQ}, one has to resort to a piece-wise definition of the $B$-operator, since we cannot rely anymore on the positivity of $V$-values.
 This causes inconvenient discontinuities in resulting policies and $V$-values.
 Furthermore, there is a problem with possible changes of the supports, and a new monotonicity proof has to account for that.

 Finally, possible support changes are problematic when proving lemma \ref{le:greedylim}, where we need to construct policy ratios in order to establish weak convergence. Here, the proof is much more complicated without the the assumption of strict positivity of the reward used in the main text (which now we are no longer assuming).

 \section{Lemmas}
 \label{ap:CTsection}

 This section contains Lemmas used in the convergence proof.

 \begin{lemma}
 \label{le:levelsets}
 (on level sets of continuous function on compact metric space)
 Let $(X,d)$ be a compact metric space and $f:X \rightarrow \mathbb{R}$ be a
 continuous function.
 Furthermore, let $m := \max_{x\in X} f(x)$ and $F := \{x \in X: f(x) = m \}$.
 Then for every open $U \subset X$, $F \subset U$ there exists a $\delta >0$ such that
 $\{x \in X : f(x) > m-\delta \} \subset U$.
 \end{lemma}
 \begin{proof}
 First notice that $m$ is defined correctly as $f$ is a continuous function on a compact space
 and therefore always has a maximum.
 Also, note that $F$ is compact and $F \neq \emptyset$.
 Assume that $f$ is not constant (otherwise the conclusion holds trivially). Now consider an open set $U$ and $F \subset U$. If $U = X$, the Lemma
 holds trivially, thus assume $U \neq X$. From compactness of $F$ we conclude that
 $F$ is $2\epsilon$ isolated from $U^C := X \setminus U$ for some $\epsilon >0$.
 Let us define $V:= \{ x\in X : d(x,F) < \epsilon \} \subset U$ an open set. Further, define
 $m':= \max f(X\setminus V)$. Notice that the definition is correct since $X\setminus V$ is closed and therefore compact and also
 $X\setminus V \neq \emptyset$ as $X\setminus V \supset U^C \neq \emptyset$.
 Further, $m' < m$ as $X\setminus V$ and $F$ are disjoint ($F \subset V$).
 Define $\delta := \frac{m-m'}{2}$. It remains to verify that $W:= \{ x \in X : f(x) > m-\delta \} = \{ x \in X : f(x) > m-\frac{m+m'}{2}\} \subset U$.
 Notice that $f(W) > \frac{m+m'}{2} > m' \geq f(X\setminus V)$. Thus $W$ and $X\setminus V$ must be disjoint
 and therefore $W \subset V \:(\subset U)$.
 \end{proof}

 \begin{lemma}
 \label{le:quotient}
 (quotient of a metric space by a compact subset)
 Let $(X,d)$ be a metric space and $F \subset X$ compact.
 Furthermore, let $\tau$ denote the topology on $X$ induced by the metric $d$.
 Define the equivalence:
 $$
 (\forall x,y \in X\times X ): x\sim y \iff (x=y \vee (x\in F \wedge  y\in F)).
 $$
 Define a (factor) quotient space $\tilde{X} := X/\sim$ and $\nu:X \rightarrow \tilde{X}$
 the canonical projection $\nu(x) := [x]_\sim$.
 \begin{description}
 \item{1.}
 Denote by $\tilde{\tau}$ the quotient topology on $\tilde{X}$ (induced by $\tau$ and $\nu$). Then it holds:
 $$
 \tilde{\tau} = \{ \nu(U) : U \in \tau, ( U \cap F = \emptyset \vee F \subset U )  \}.
 $$
 \item{2.}
 Further, the function
 $\tilde{d}: \tilde{X}\times \tilde{X} \rightarrow \mathbb{R}^+$
 $$
 \tilde{d}([x]_\sim,[y]_\sim) := d(x,y) \wedge (d(x,F)+d(y,F))
 $$
 defines a metric on $\tilde{X}$ and  the
 topology induced by metric $\tilde{d}$ agrees with $\tilde{\tau}$.
 \item{3.} (continuous functions)
 Let $\tilde{f}: \tilde{X} \rightarrow \mathbb{R}$ be a function on $\tilde{X}$. Than it holds:
 $$
 \tilde{f} \in C(\tilde{X}) \iff \tilde{f} \circ \nu \in C(X),
 $$
 so there is a one to one correspondence between continuous functions on $\tilde{X}$ ($C(\tilde{X}$))
 and continuous functions on $X$, which are constant on $F$ (which allow factorisation through $\nu$):
 \begin{multline*}
 \{ f \in C(X) : \exists \tilde{f} \in \mathbb{R}^{\tilde{X}} : f= \tilde{f} \circ \nu \}
 \\
 =
 \{ f \in C(X) : \exists c_f \in \mathbb{R} : f|_{F} = c_f\}.
 \end{multline*}
 \end{description}
 \end{lemma}
 Although this result is quite standard and any general topology textbook (e.g. \citet{munkres2000topology}) can serve as
 a reference here, we decided to include also the proof for convenience and completeness.
 The fact that $\tilde{X}$ is a metric space (point (2) of the Lemma)
 is necessary for the application of Portmanteau theorem in the Lemma below.
 Since the explicit form of the metric $\tilde{d}$ (given in point (2))  will not be used anywhere,
 one can also proceed by utilizing metrization theorems. Since the Lemma is intended just for the case $X=\mathcal{A}\:( \subset \mathbb{R}^{n_A})$ (the action space), which is separable, both Uryshon and Nagata-Smirnov metrization theorems  \citep{munkres2000topology} can be used.
 \begin{proof}
 During the proof, we will assume $F\neq \emptyset$. For the case $F = \emptyset$, the Lemma holds trivially.

 \textbf{1.} The quotient topology $\tilde{\tau}$ is the finest topology in which is $\nu$ continuous.
 Suppose $\tilde{U} \in \tilde{\tau}$ (is open in $\tilde{\tau}$) then $U:=\nu^{-1}(\tilde{U})$
 must be open (otherwise $\nu$ would not be continuous). Further,
 due to the equivalence defined, the pre-images under $\nu$ cannot contain $F$ only partially.
 They either contain the whole $F$, or are disjoint with $F$ (in the first case we get $F \subset U$ and
 in the second one we get $F \cap U = \emptyset$).
 This gives us the inclusion $\tilde{\tau} \subset \{ \nu(U) : U \in \tau, ( U \cap F = \emptyset \vee F \subset U )  \}$.
 For the reverse inclusion, assume we have $U \in \tau$.
 Assume $F \subset U$. Then the pre-image $\nu^{-1}(\nu(U)) = U$ (the result would be different from $U$ just
 when $U$ includes $F$ only partially), which is an open set. Thus, from the fact that $\tilde{\tau}$ is the finest topology in
 which $\nu$ is continuous, it follows that
 $\nu(U) \in \tilde{\tau}$. Similarly for $U \cap F = \emptyset$.

 \textbf{2.}
 Now we aim to show that $\tilde{d}$ is a metric on $\tilde{X}$. Notice that the definition is
 correct in the sense that it does not depend on the choice of representants.
 When we assume that both $x,y$ are not in $F$, then the choice of representants is unique.
 So assume that, for example, $x\notin F$,$y\in F$. Then we can choose another representant for $[y]_\sim$,
 but then $\tilde{d}([x]_\sim,[y]_\sim) = d(x,F)$ is independent of $y$.
 Similarly, if  $x,y$ are both in $F$ then $\tilde{d}([x]_\sim,[y]_\sim) = 0$ which again
 does not depend on choice of the representants.
 Non-negativity and symmetry trivially holds.
 First, we consider the property:
 $$
 \tilde{d}([x]_\sim,[y]_\sim) = 0 \iff [x]_\sim = [y]_\sim  ( \iff x \sim y ).
 $$
 Assume $x \sim y$, then either $x=y$ or $x,y \in F$.
 In both cases $\tilde{d}([x]_\sim,[y]_\sim)$ becomes zero.
 Assume $\tilde{d}([x]_\sim,[y]_\sim) = 0$, then $d(x,y) = 0$ or
 $d(x,F)+d(y,F) = 0$, where in the first case we get $x=y$ and in the second case
 (here we use that $F$ is closed) $x,y \in F$. Thus $x\sim y$. The Triangle inequality holds too. The proof follows easily, but is omitted for brevity (it consists of checking multiple cases).

 Finally, we have to show that the topology induced by $\tilde{d}$ agrees with
 $\tilde{\tau}$ (here we will need compactness of $F$).
 First we show that every open set in $\tilde{\tau}$ is also open in the topology induced by
 $\tilde{d}$.
 Let us consider an open set $\tilde{U} \in \tilde{\tau}$. Now let us fix an arbitrary point
 $\tilde{x} \in \tilde{U}$. It suffices to show that there exits $r >0$ such that open ball
 $U_r(\tilde{x}) := \{ \tilde{y} \in \tilde{X} :  \tilde{d}(\tilde{x},\tilde{y}) < r\}
 \subset \tilde{U}$.
 From $\tilde{U} \in \tilde{\tau}$ there exists $U \in \tau$ such that $\nu(U) = \tilde{U}$
 and moreover $F \subset U$ or $F \cap U = \emptyset$.

 Fix $x \in X$ such that $[x]_\sim = \tilde{x}$.
 We start by considering the case $F \subset U$ and $x \in F$.
 Notice that the metric reduces to $\tilde{d}([x]_\sim,[y]_\sim) = d(y,F)$.
 Compactness of $F$ guarantees that there exists $\epsilon > 0$ such that $F$ is
 $\epsilon$ isolated from $U^c := X \setminus U$. So it suffices to choose $r:=\epsilon$.

 For the second case we consider $F \subset U$ and $x \notin F$.
 As $U \setminus F$ is open, there exists a $\delta >0$ such that
 $U_{\delta}(x) := \{ y \in X : d(x,y) < \delta \} \subset U \setminus F$.
 Note that $\nu(U_{\delta}(x))$ is an open set in $\tilde{\tau}$
 (has open pre-image and does not contain $F$) on which the metric simplifies
 to $\tilde{d}([x]_\sim,[y]_\sim) = d(x,y)\quad ( < \delta)$. We conclude that it is an open
 ball in $\tilde{d}$, whole lying in $\tilde{U}$. So it suffices to put $r:=\delta$.

 As final case, assume $F \cap U = \emptyset$. This actually reduces to the second case
 we already considered.

 Finally, for the opposite inclusion it suffices to show that
 every open ball in $\tilde{d}$ is an open set in $\tilde{\tau}$.
 Thus let us fix an $x \in X$ and positive $r>0$ and set $\tilde{U} := U_r(\tilde{x})$.
 In order for $\tilde{U}$ to be open in $\tilde{\tau}$, it must have open pre-image
 \begin{multline*}
 \nu^{-1}(\tilde{U}) =  \{y \in X : \nu(y) \in \tilde{U} \} =\\
 \begin{aligned}
 &=
 \{y \in X : d(x,y) \wedge (d(x,F) + d(y,F)) < r\}
 \\&=
 \{y \in X : d(x,y)  < r\}
 \\&\quad\cup
 \{y \in X : (d(x,F) + d(y,F)) < r\},
 \end{aligned}
 \end{multline*}
 where we end up with a union of two sets, both open in $\tau$, which is again open.
 Thus $\nu^{-1}(\tilde{U})$ is open, so $\tilde{U}$ is open (from $\tilde{\tau}$
 is the finest topology in which $\nu$ is continuous).

 \textbf{3.}
 (Continuous functions)
 Assume $\tilde{f} \in C(\tilde{X})$. Since $\nu$ is continuous, then $\tilde{f} \circ \nu$ is continuous (composition of continuous maps). For the opposite implication, assume $f := \tilde{f} \circ \nu$ is continuous. We have to show that $\tilde{f}$ is continuous. Thus fix an arbitrary open set $V \subset \mathbb{R}$. We have to show that the pre-image $\tilde{U} := \tilde{f}^{-1}(V)$ is open. We know that $U := f^{-1}(V)$ is open from the continuity of $f$ and that $U = f^{-1}(V) = \nu^{-1}(\tilde{U})$,
 that means that the pre-image of $\tilde{U}$ under $\nu$ is open, but $\tilde{\tau}$ is the finest topology in which $\nu$ is continuous, therefore $\tilde{U}$ has to be open.
 \end{proof}

 \begin{lemma}
 \label{le:RPort}
 (Adaptation of Portmanteau theorem conditions to relative weak convergence)
 Let $(X,d)$, $(\tilde{X},\tilde{d})$, $F$, $\nu$ be like above.
 Let $P, P_n, n\in\mathcal{N}$ be probability measures on $\mathcal{B}(X)$.
 Then following conditions are equivalent:
 \begin{description}
 \item{1.}
 $
 P_n \rightarrow^{w(F)} P.
 $
 \item{2.}
 For all continuous $f:X\rightarrow \mathbb{R}$ that are constant on $F$ it holds that
 $
 P_n f \rightarrow P f.
 $
 \item{3.}
 For all $U \subset X$ open satisfying $U \cap F = \emptyset$ or $F \subset U$
 it holds that
 $
 \liminf P_n U \geq P U.
 $
 \end{description}
 \end{lemma}

 \begin{proof}
 First we show equivalence of 1. and 2.
 Point 1. is equivalent to
 $
 \nu P_n \rightarrow^{w} \nu P,
 $ (definition \ref{de:Rweakconv})
 which is equivalent to (using Portmanteau theorem):
 $$
 (\forall \tilde{f} \in C(\tilde{X})) : (\nu P_n) \tilde{f} \rightarrow^{} (\nu P) \tilde{f},
 $$
 what can be rewritten using definition of image measure:
 $$
 (\forall \tilde{f} \in C(\tilde{X})) :
 P_n (\tilde{f}\circ\nu) \rightarrow^{} P (\tilde{f}\circ\nu).
 $$
 But from Lemma \ref{le:quotient}
 we already know that there is a one to one
 correspondence between functions in $C(\tilde{X})$ and functions in $C(X)$, which factors
 through $\nu$ (are constant on $F$). Thus it is equivalent to:
 \begin{multline*}
 (\forall f \in C(X)) : \\
 ((\exists c_f \in \mathbb{R}): f|_F = c_f) \implies ( P_n f \rightarrow P f).
 \end{multline*}

 Finally, we show equivalence of 1. and 3.
 Again, point 1. is equivalent (using Portmanteau theorem) to:
 $$
 (\forall \tilde{U} \subset \tilde{X} \:\text{open}) :
 \liminf (\nu P_n) \tilde{U}  \geq (\nu P) \tilde{U}.
 $$
 Using the definition of image measure and the one to one correspondence (see Lemma \ref{le:quotient}) between all open sets in $\tilde{X}$
 and  open sets in $X$ we have that at least one of the two
 conditions $U \cap F = \emptyset$, $F \subset U$ is satisfied. This concludes the result.
 \end{proof}

 \section{Convergence order - finite case}
 \label{se:crates}

 \rlinearconv*

 \begin{proof}
 We will use the upper bound on
 $(\pi_n(\cdot|s)\cdot\mu_A)(A_{\epsilon})$
 constructed in lemma \ref{le:greedylim} for the choice of
 $\epsilon$ below. In the following, $A_{\epsilon}$,$B_{\epsilon}$,$\alpha$,
 $c_m$ and  $n_0$ are defined as in lemma \ref{le:greedylim}.
 Fix an arbitrary $s \in \mathcal{S}$. Like in the lemma \ref{le:greedylim}, for the first part of this proof the dependence on $s$ will not be made explicit.
 Let $\lambda_0 := \max_{a \in \mathcal{A} \setminus M(s)} \lambda(a)$, where the function $\lambda \leq 1$ is defined as in lemma \ref{le:greedylim}.
 Notice that $\lambda_0 < 1$. Indeed, if $\lambda_0 = 1$ then there
 exists $a_0 \in \mathcal{A} \setminus M(s)$ such that
 $\lambda_0(a_0) = 1$ meaning that $a_0 \in M(s)$, which is a contradiction.
 Now define $\epsilon$ such that $2\epsilon = 1-\lambda_0$. Then $A_{\epsilon} = \mathcal{A} \setminus M(s)$, $B_{\epsilon} = M(s)$, and $\alpha = \frac{2\lambda_0}{0.9+1.1\lambda_0} < 1$.

 Furthermore, using the upper bound from lemma \ref{le:greedylim} we have:
 \begin{equation}
 (\pi_n(\cdot|s)\cdot\mu_A)(\mathcal{A} \setminus M(s)) \leq \alpha^n C,
 \text{for}\:n > n_0
 \label{eq:rate1}
 \end{equation}
 where
 $
 C := c_m
 \frac{\mu_A A_{\epsilon}}{\mu_A B_{\epsilon}}
 $.
 Notice that
 $(\pi_n(\cdot|s)\cdot\mu_A)(B_{\epsilon}) \leq 1$.

 Now we drop the assumption of fixed $s\in \mathcal{S}$
 and we denote the dependence on $s$ explicitly. Starting from Eq. \eqref{eq:rate1} we have that $\forall s \in \mathcal{S}$:
 \begin{equation}
 (\pi_n(\cdot|s)\cdot\mu_A)(\mathcal{A} \setminus M(s))
  \leq \alpha^n(s) C(s) \leq \alpha_m^n C_m,
 \label{eq:rate2}
 \end{equation}
 for $n > n_m$,
 where $\alpha_m := \max_{s \in \mathcal{S}} \alpha(s) < 1$,
 $C_m := \max_{s \in \mathcal{S}} C(s)$ and $n_m := \max_{s \in \mathcal{S}} n(s)$.

 Now we have that $\forall s \in \mathcal{S}$:
 \begin{multline}
 |V^*(s) - V_n(s)|  \leq \\
 \begin{aligned}
 &|\int_{\mathcal{A}} Q^*(s,a) \dif\pi^*(a|s)
 - \int_{\mathcal{A}} Q^*(s,a) \pi_n(a|s) \dif a|+\\
 & |\int_{\mathcal{A}} Q^*(s,a) \pi_n(a|s) \dif a
 - \int_{\mathcal{A}} Q_n(s,a) \pi_n(a|s) \dif a|
 .
 \end{aligned}
 \label{eq:rate3}
 \end{multline}

 Using Eq. \eqref{eq:rate2} and the following facts:
 $$
 \begin{gathered}
 \int_{\mathcal{A} \setminus M(s)} Q^*(s,a) \dif\pi^*(a|s) = 0,\\
 V^*(s) = \int_{M(s)} Q^*(s,a) \dif\pi^*(a|s) = \max_{a\in \mathcal{A}}
 Q^*(s,a)
 \end{gathered}
 $$
 the first expression in Eq. \eqref{eq:rate3} can be bounded $\forall s \in \mathcal{S}$:
 \begin{multline*}
 |\int_{\mathcal{A}} Q^*(s,a) \dif\pi^*(a|s)
 - \int_{\mathcal{A}} Q^*(s,a) \pi_n(a|s) \dif a| \\
 \begin{aligned}
 &\leq
 \Big|\int_{\mathcal{A} \setminus M(s)} Q^*(s,a) \dif\pi^*(a|s)
 \\&\quad\quad\quad\quad- \int_{\mathcal{A} \setminus M(s)} Q^*(s,a) \pi_n(a|s) \dif a\Big|
 \\&\quad+\Big|\int_{M(s)} Q^*(s,a) \dif\pi^*(a|s)
 \\&\quad\quad\quad\quad- \int_{M(s)} Q^*(s,a) \pi_n(a|s) \dif a\Big|
 \\&\leq
 \int_{\mathcal{A} \setminus M(s)} Q^*(s,a) \pi_n(a|s) \dif a
 \\&\quad+ V^*(s) - \int_{M(s)} Q^*(s,a) \pi_n(a|s) \dif a
 \\&\leq
 V^*(s) (\pi_n(\cdot|s)\cdot\mu_A)(\mathcal{A} \setminus M(s))
 \\&\quad+ V^*(s) - V^*(s)(\pi_n(\cdot|s)\cdot\mu_A)(M(s))
 \\&=
 V^*(s) (\pi_n(\cdot|s)\cdot\mu_A)(\mathcal{A} \setminus M(s))
 \\&\quad+ V^*(s) - V^*(s)(1-(\pi_n(\cdot|s)\cdot\mu_A)(\mathcal{A} \setminus M(s)))
 \\&=
 2V^*(s) (\pi_n(\cdot|s)\cdot\mu_A)(\mathcal{A} \setminus M(s))
 \\&\leq
 2V_m \alpha_m^n C_m
 \end{aligned}
 \end{multline*}
 for all $n > n_m$, where $V_m = \max_{s\in \mathcal{S}} V^*(s)$.

 The second expression in Eq. \eqref{eq:rate3} can be bounded $\forall s \in \mathcal{S}$:
 \begin{multline*}
 |\int_{\mathcal{A}} Q^*(s,a) \pi_n(a|s) \dif a
 - \int_{\mathcal{A}} Q_n(s,a) \pi_n(a|s) \dif a| \\
 \leq
 \| Q^*-Q_n\|_{\infty}
 \leq
 \| W \|_{\infty} \| V^*-V_n\|_{\infty}
 =
 \gamma \| V^*-V_n\|_{\infty}
 \end{multline*}
 Combining everything together:
 $$
 \| V^*-V_n\|_{\infty}
 \leq
 2V_m \alpha_m^n C_m + \gamma \| V^*-V_n\|_{\infty},
 $$
 which is equivalent to:
 $$
 \| V^*-V_n\|_{\infty}
 \leq \alpha_m^n \frac{2V_m C_m}{1-\gamma}
 .
 $$
 Thus $\| V^*-V_n\|_{\infty} = O(\alpha_m^n)$. Finally, the formulas for $\alpha_m$ and $\lambda_m$ can be obtained from pure substitution.
 \end{proof}

 The result in the corollary can be developed in a slightly more general case than finite state and action spaces. It suffices to assume that $M(s)$ is open (i.e. clopen) for all $s \in \mathcal{S}$ and then to show continuity of e.g. $s \mapsto C(s)$, $s \mapsto \alpha(s)$
 etc. This makes the proof more technical. The development for not open $M(s)$ (general case) is more complex and it is out of scope
 of this paper.

 We can observe that for the finite case $\| V^*- V_n \|_{\infty}$
 converges to $0$ R-linearly, i.e. $\| V^*- V_n \|_{\infty}$
 is bounded by a Q-linearly converging sequence $\alpha_m^n$.
 We finish this section with two examples. The example for a finite MDP introduced in lemma \ref{le:fexample} below
 which exhibits linear convergence rate
 demonstrates that the upper bound from the corollary \ref{le:ratescol}
 is asymptotically tight (i.e. could not be improved) in sense of order of convergence. Therefore a linear order is the best we can achieve. The second example in lemma \ref{le:cexample} further shows that convergence orders encountered in the "continuous case"
 could be much slower (i.e. sub-linear).

 \begin{lemma}(finite example with Q-linear rate)
 \label{le:fexample}
 Consider a two states MDP with one initial state and one goal state. When the goal state is entered the agent stays there forever, independently of action taken.
 In the initial state there is a possibility to chose form two actions
 $\mathcal{A} =\{0,1\}$ which both transit to the goal state.
 The action value function in the initial state is given
 by $q_0  := Q(a=0) := 2, q_1  := Q(a=1) := 1$.

 The space of all policies can be parametrized by the probability of
 the first action $p \in [0,1]$ in the initial state.
 $B$ then can be interpreted as a map $B:[0,1] \rightarrow [0,1]$
 given by
 $$
 B(p) := \frac{p q_0}{p q_0 + (1-p) q_1} = \frac{2p}{1+p}.
 $$

 The following holds:
 \begin{description}
 \item{1.}
 The optimal policy is given by the value $p^* = 1$ of the parameter $p$.
 $B(p) \geq p$ on $p \in [0,1]$ and for $p \in (0,1)$ the inequality
 is strict.
 \item{2.}
 For $0 \leq c \leq 1$ the map $B$ is $\frac{1}{1+c}$-Lipschitz (it is a contraction for $c > 0$)
 on $[c,1]$ with respect to the $|\cdot|$-norm (the norm given by the absolute value), i.e.
 $$
 \max_{p \in [c,1]} \frac{|p^* - Bp|}{|p^* - p|} = \frac{1}{1+c}.
 $$
 Further:
 $$
 \lim_{n \rightarrow \infty}
 \frac{p^* - p_{n+1}}
 {p^* - p_{n}}
 =\frac{1}{2}
 ,
 $$
 where $p_n := B^{\circ n} p_0$ (n-th application of $B$ on some initial condition $0< p_0 < 1$), i.e.
 $p_n = B^{\circ n} p_0$
 converges to  $\rightarrow p^* = 1$ as $n \rightarrow \infty$ Q-linearly (the order of convergence is 1)
 with convergence rate $\frac{1}{2}$.
 The value function $V_n$ (for the policy parametrized by $p_n$)
 converges to the optimal $V^*$ exactly the same way, i.e.
 Q-linearly with convergence rate $\frac{1}{2}$.
 \item{3.}
 Define a strictly decreasing function
 $f:(0,1] \rightarrow (0,1], f(x) := \frac{1-x}{x}$.
 and a metric $d:(0,1]^2 \rightarrow \mathbb{R}^+$, $d(x,y) := |f(x)-f(y)|$. Then $B$ is a $\frac{1}{2}$-contraction on $(0,1]$.
 \end{description}
 \end{lemma}
 \begin{proof}
 \textbf{1.} The optimal policy is derived directly from definition of action value function $Q$ and the parameter $p$. Further, $Bp - p = \frac{p(1-p)}{1+p} \geq 0 $ on $[0,1]$ and the strict inequality holds by inspection for $p \in (0,1)$.\\
 \textbf{2.}
 We have that:
 \begin{multline*}
 \max_{p \in [c,1]} \frac{|p^* - Bp|}{|p^* - p|}
 =
 \max_{p \in [c,1]} \frac{|1- \frac{2p}{1+p}|}{|1 - p|}
 \\
 =
 \max_{p \in [c,1]} \frac{|\frac{1-p}{1+p}|}{|1 - p|}
 =
 \max_{p \in [c,1]} \frac{1}{1 + p}
 =
 \frac{1}{1 + c}
 \end{multline*}
 A similar approach is used to prove Q-linear convergence:
 $$
 \lim_{n \rightarrow +\infty}
 \frac{p^* - p_{n+1}}
 {p^* - p_{n}}
 =
 \lim_{p \rightarrow 1-}
 \frac{1 - B(p)}
 {1 - p}
 =
 \lim_{p \rightarrow 1-}
 \frac{\frac{1-p}{1+p}}{1 - p}
 =
 \frac{1}{2}
 .
 $$
 Regarding the convergence of value functions, we have that:
 \begin{multline*}
 V^* - V_n = q_0 - ( q_0 p_n + q_1 (1-p_n) ) =
 \\
 2 - 2 p_n - (1-p_n) = 1- p_n = p^* - p_n
 .
 \end{multline*}
 Therefore the convergence must be exactly the same as for the policies.\\
 \textbf{3.}
 It is easy to verify that $d$ is a metric, e.g. for
 the triangle inequality one can easily see that
 $|f(x)-f(y)| < |f(x)-f(z)| + |f(z)-f(y)|$ for all
 $x,y,z \in (0,1]^3$. In order to prove the contraction property, note that $p=1$ is a fixed point for
 $p \in (0,1)$:
 \begin{multline*}
 \frac{d(p^*,Bp)}{d(p^*,p)}
 =
 \frac{|f(p^*) - f(Bp)|}{|f(p^*) - f(p)|}
 \\
 =
 \frac{|f(Bp)|}{|f(p)|}
 =
 \frac{|\frac{1-\frac{2p}{1+p}}{\frac{2p}{1+p}}|}{|\frac{1-p}{p}|}
 =
 \frac{\frac{1-p}{2p}}{\frac{1-p}{p}}
 =
 \frac{1}{2}.
 \end{multline*}
 \end{proof}

 \begin{lemma}(example in "continuous" action space)
 \label{le:cexample}
 Let us have a two state MDP with one initial state and one goal state. When the goal state is entered the agent stays there forever independent of action taken.
 In the initial state there is a possibility to chose from $\mathcal{A} =[0,1]$ actions which all transit to the goal state.
 The action value function in the initial state is given by $Q^*(a) = a+1$. Assume the reference measure $\mu_A$ is the Lebesgue measure.
 Assume the initial policy density in the initial state to be $\pi_0 = 1$, i.e. the initial policy in the initial state is uniform.\\
 Then the optimal policy measure in the initial state is given by
 $\pi^*(B) = \Indicator_B(1)$ for all $B \in \mathcal{B}(\mathcal{A})$,
 the optimal state-value function in the initial state is $V^* = 2$,
 the policy density in the initial state is $\pi_n = \frac{(n+1)(a+1)^n}{2^{n+1}-1}$ and
 the state-value function in the initial state is
 $V_n = \frac{n+1}{n+2}\frac{2^{n+2}-1}{2^{n+1}-1}$.
 Further,
 $
 | V^* - V_n | = \frac{2^{n+2}-n-3}{(n+2)(2^{n+1}-1)} = \Theta(n^{-1})
 $
 and
 $
 d(\pi^*,\pi_n\cdot\mu_A) = \frac{2^{n+2}-n-3}{(n+2)(2^{n+1}-1)} = \Theta(n^{-1}),
 $
 where $d(P,Q) :=  \sup \{ |Pl -Ql| :  l \in BL(\mathcal{A}), \|l\|_{BL(\mathcal{A})} \leq 1 \}$ for $P,Q$ measures on $\mathcal{A}$,
 where $BL(\mathcal{A})$ stands for the space of all Bounded Lipschitz functions
 on $\mathcal{A}$
 with norm $\|l\|_{BL(\mathcal{A})} := \max \{K_1, 2K_2\}$,
 $K_1 := \sup_{x\neq y} \frac{|l(x)-l(y)|}{|x-y|}$,
 $K_2 := \sup_{x} |l(x)|$.
 The $d$ is a metric
 inducing the same topology on the space of all policy measures (in the initial state)
 as relative (to $M := \argmax Q^* = \{1\}$) weak convergence does.
 \end{lemma}

 \begin{proof}
 The formula for $\pi^*$ is obtained from the fact that the optimal policy measure must concentrate all the mass on $\argmax Q^*$. $V^* := \pi^*(Q^*) = \int_{\mathcal{A}} Q^*(a) \dif\pi^*(a) = 2$.
 Further, we know that $\pi_n \propto (Q^*)^n \pi_0 = (a+1)^n$, where the normalization factor is $\int_{\mathcal{A}} (a+1)^n \dif a = \frac{2^{n+1}-1}{n+1}$, which gives the formula for
 $\pi_n(a) = \frac{(a+1)^n}{\int_{\mathcal{A}} (a+1)^n \dif a} = \frac{(n+1)(a+1)^n}{2^{n+1}-1}$.
 For the state-value function we obtain:
 $V_n = \int_{\mathcal{A}} Q^*(a) \pi_n(a) \dif a = \frac{(n+1)}{2^{n+1}-1}\int_{\mathcal{A}}  (a+1)^{n+1}\dif a  = \frac{n+1}{n+2}\frac{2^{n+2}-1}{2^{n+1}-1}$.
 From
 \begin{align*}
 |V^* - V_n|
 &= \frac{(n+2)(2^{n+2}-2)-(n+1)(2^{n+2}-1)}{(n+2)(2^{n+1}-1)}
 \\&= \frac{2^{n+2}-n -3}{(n+2)(2^{n+1}-1)}
 \\&= \frac{2}{n}\frac{1-\frac{n+3}{2^{n+2}}}{(1+\frac{2}{n})(1 - \frac{1}{2^{n+1}})}
 \end{align*}
 we see that $|V^* - V_n| = \Theta(n^{-1})$ as
 $$
 \lim_{n \rightarrow \infty} \frac{|V^* - V_n|}{n^{-1}} = 2 \in (0,+\infty).
 $$

 Now we prove the statement about convergence of policy measures $\pi_n\cdot\mu_A$.
 Notice that in the goal state any policy measure is optimal, therefore we can
 resort to discuss convergence just in the initial state.
 Further, the factorisation induced by $M:= \argmax Q^* = \{1\}$ in the definition of relative
 weak convergence has no effect here as $M$ consists of a single point.
 Thus effectively we are left with just ordinary weak convergence of measures
 on a compact separable metric space $\mathcal{A}$.
 This convergence can be described by
 the metric $d$ introduced in the lemma, see \cite{pollard2001user}.
 Both the metric $d$ and weak convergence induce the same topology.
 During the proof we will use the following fact, see \cite{pollard2001user} ($\forall l \in BL(\mathcal{A}), x \in \mathcal{A}, y \in \mathcal{A}$)
 \begin{equation}
 |l(x)-l(y)| \leq \| l \|_{BL(\mathcal{A})} (\min\{ 1, |x-y|\})
 .
 \label{eq:cexampleBL}
 \end{equation}
 Let us first upper bound $d(\pi^*,\pi_n\cdot\mu_A)$:
 \begin{equation}
     d(\pi^*, \pi_n\cdot\mu_A) = \sup_{l \in BL(\mathcal{A}) : \|l\|_{BL(\mathcal{A})} \leq 1} |\pi^*l - (\pi_n\cdot\mu_A) l|
 \end{equation}
 For any $l \in BL(\mathcal{A})$, $\|l\|_{BL(\mathcal{A})} \leq 1$ it holds:
 \begin{multline*}
 |\pi^*l - (\pi_n\cdot\mu_A) l| \\
 \begin{aligned}
 &=
 |l(1) - \int_{\mathcal{A}} l(a) \pi_n(a) \dif a |
 \\&\leq
 \int_{\mathcal{A}} | l(1) - l(a) | \pi_n(a) \dif a
 \\&\leq
 \int_{\mathcal{A}} \| l \|_{BL(\mathcal{A})} (\min \{1,|1-a|\})  \pi_n(a) \dif a
 \\&\leq
 \int_{\mathcal{A}} (1-a)  \pi_n(a) \dif a
 \\&=
 \int_{\mathcal{A}} (2-(a+1))  \frac{(n+1)(a+1)^n}{2^{n+1}-1} \dif a
 \\&=
 \frac{n+1}{2^{n+1}-1}
 (
 2
 \int_{\mathcal{A}} (a+1)^n \dif a
 -
 \int_{\mathcal{A}}  (a+1)^{n+1} \dif a
 )
 \\&=
 \frac{n+1}{2^{n+1}-1}
 (
 2
 \frac{2^{n+1}-1}{n+1}
 -
 \frac{2^{n+2}-1}{n+2}
 )
 \\&=
 \frac{2^{n+2}-n -3}{(n+2)(2^{n+1}-1)}
 .
 \end{aligned}
 \end{multline*}
 Thus $d(\pi^*,\pi_n\cdot\mu_A) \leq \frac{2^{n+2}-n -3}{(n+2)(2^{n+1}-1)}$.
 Now we show that the given bound is tight, it suffices to consider
 $l = a$:
 \begin{align*}
 |\pi^*l - (\pi_n\cdot\mu_A) l|
 &=
 |1 - \int_{\mathcal{A}} a \pi_n(a) \dif a |
 \\&=
 \int_{\mathcal{A}} (1 - a) \pi_n(a) \dif a
 \\&=
 \frac{2^{n+2}-n -3}{(n+2)(2^{n+1}-1)}
 .
 \end{align*}

 Thus we are left with
 $d(\pi^*,\pi_n\cdot\mu_A) = \frac{2^{n+2}-n -3}{(n+2)(2^{n+1}-1)} = \Theta(n^{-1})$.
 \end{proof}

 \section{Motivation of approach}
 \label{se:relevance}

 Here we provide a justification on why we cover the general case of compact state and action spaces and not just the particular case of finite states and actions. Moreover, we comment on the difficulties of using Banach fixed point theorem in our setting.

 \paragraph{Why studying just finite case is not enough}
 The main reason why we study our problem for compact state and action spaces is that we want to cover also the robotic control scenario, which is of great importance today and involves multidimensional "continuous" state and action spaces.
 One could wonder if our results could be easily studied in the finite setting and then extended to the compact case.
 However, the example in lemma \ref{le:fexample} (finite case) proved $O(\alpha_m^n)$ convergence of state-value function to the optimum,
 while the example in lemma \ref{le:cexample} ("continuous" case) showed much slower $O(n^{-1})$ convergence.
 Therefore the intuition coming from the finite case does not apply to the "continuous" one.
 In general, one can always approach the continuous case by discretization. However there is always a discretization error involved which is difficult to study.

 Considering directly the general compact setting avoids this problem, although it necessarily involves measure and topology arguments.

 \paragraph{Why we do not employ Banach fixed point argument}

 Using a Banach contraction argument could simplify the proof a lot. However, it is hard to make the $B$ operator a contraction on a
 complete metric space (these are the assumptions of
 Banach fixed point theorem) in our compact setting.
 Some insights are provided in our examples in Lemmas
 \ref{le:fexample} and \ref{le:cexample}.
 Lemma \ref{le:fexample} demonstrates that the
 $B$ operator is not a contraction on $\mathcal{A}$.
 One has to remove non-optimal deterministic policies with some open neighbourhoods so the resulting space becomes complete again (see point \textbf{2} of the example). This removing has to be performed carefully because the resulting space must be closed under the $B$ operator. In general, it is not trivial to close it under $B$ again since the union of closed sets is not generally closed. This could work for the finite case but the "continuous case" from the example in Lemma \ref{le:cexample} exhibits asymptotic behaviour $\Theta(n^{-1})$ which is not sufficient in order for $B$ to be a contraction. Therefore, removing non-optimal deterministic policies is not enough. One could try to distort the metric purposefully (like in the example in Lemma
 \ref{le:fexample} in point \textbf{3.}), although it is not clear how much this would complicate the proof. The approach we used in the paper appears to be more straightforward.

 \section{Computational Requirements of Demonstration}

 The source code for the demonstration in Section~\ref{sec:exp} is available at \url{https://github.com/dylanashley/reward-weighted-regression/releases/tag/v1.0.0}. The plot shown in Figure~\ref{fig:results} was generated using the source code as executed by \texttt{Python 3.8.11}. The computational requirements of this were are minimal, and generating the plot again from scratch should take under an hour on most modern personal computers.

\end{document}